\documentclass[11pt]{article}

\usepackage[utf8]{inputenc} 
\usepackage[T1]{fontenc}    

\usepackage{amsmath,amsthm}
\usepackage{subcaption}
\usepackage{newtxtext}
\usepackage{newtxmath}
\usepackage{bm}
\usepackage{natbib}
\usepackage{booktabs}       
\usepackage{nicefrac}       
\usepackage{microtype}      
\usepackage{xcolor}         
\usepackage{comment}
\usepackage[algoruled]{algorithm2e}
\usepackage{multirow}
\usepackage{array}
\usepackage[margin=1in]{geometry}
\usepackage{graphicx}

\newcommand{\Rb}{\mathbb{R}}      
\newcommand{\Pb}{\mathbb{P}}
\newcommand{\Qb}{\mathbb{Q}}
\newcommand{\Sb}{\mathbb{S}}
\newcommand{\Ab}{\mathbb{A}}
\newcommand{\Xb}{\mathbb{X}}

\newcommand{\Vb}{\mathbb{V}}
\newcommand{\Zb}{\mathbb{Z}}
\newcommand{\Ac}{\mathcal{A}}

\newcommand{\Eb}{\mathbb{E}}
\newcommand{\Fc}{\mathcal{F}}

\newcommand{\Nc}{\mathcal{N}}
\newcommand{\Oc}{\mathcal{O}}
\newcommand{\Pc}{\mathcal{P}}

\newcommand{\Vc}{\mathcal{V}}

\newcommand{\1}{\varmathbb{1}}
\newcommand{\0}{\varmathbb{0}}

\newcommand{\argmin}{\mathop{\rm argmin}}
\newcommand{\eqdef}{\triangleq}
\newcommand{\D}{\mathop{\text{\rm d}\!}}

\newcommand{\Cdot}{\,\cdot\,}

\newenvironment{tightitemize}{%
    \list{{\textup{$\bullet$}}}{\settowidth\labelwidth{{\textup{\qquad}}}
    \leftmargin\labelwidth \advance\leftmargin\labelsep
    \parsep 0pt plus 1pt minus 1pt \topsep 3pt \itemsep 3pt
    }}{\endlist}

 \makeatletter
\providecommand*\rel@kern[1]{\kern#1\dimexpr\macc@kerna}

\renewcommand*\widebar[1]{%
  \begingroup
  \def\mathaccent##1##2{%
    \rel@kern{0.8}%
    \overline{\rel@kern{-0.8}\macc@nucleus\rel@kern{0.2}}%
    \rel@kern{-0.2}%
  }%
  \macc@depth\@ne
  \let\math@bgroup\@empty \let\math@egroup\macc@set@skewchar
  \mathsurround\z@ \frozen@everymath{\mathgroup\macc@group\relax}%
  \macc@set@skewchar\relax
  \let\mathaccentV\macc@nested@a
  \macc@nested@a\relax111{#1}%
  \endgroup
}
\makeatother

\newtheorem{theorem}{Theorem}
\newtheorem{definition}[theorem]{Definition}
\newtheorem{proposition}[theorem]{Proposition}

\newtheorem{lemma}[theorem]{Lemma}

\title{Reinforcement Learning
with Markov Risk Measures\\
and Multipattern Risk Approximation}

\author{Andrzej Ruszczy\'nski  and Tiangang Zhang\footnote{Department of Management Science and Information Systems, Rutgers University, email: rusz@rutgers.edu; tiangangzh@gmail.com}
}

\date{December 1, 2025; revised November 1, 2026}

\begin{document}

\maketitle

\begin{abstract}
For a risk-averse finite-horizon Markov Decision Process, we introduce a special class of Markov coherent risk measures, called mini-batch measures. We also define the class of multipattern risk-averse  problems that reduce to linear MDPs in the risk-neutral case. We use both concepts in a feature-based $Q$-learning method with multipattern $Q$-factor approximation and we
prove a high-probability regret bound of $\widetilde{\mathcal{O}}\big(d^{3/2}HN^{H-1}\sqrt{K}\big)$, where $d$ is the number of features, $H$ is the horizon, $N\ge2$ is the mini-batch size, $K$ is the number of episodes, and $\widetilde{\mathcal{O}}$ hides logarithmic factors. We also propose an economical version of
the $Q$-learning method that streamlines the policy evaluation (backward) step. The theoretical results are illustrated on a stochastic assignment problem and a short-horizon multi-armed bandit problem.
\end{abstract}





\section{Introduction}
We consider a finite-horizon Markov Decision Process (MDP) with the state space $\Xb$, the action space $\Ab$,
the controlled transition kernels $\Pb_h:\Xb\times\Ab \to \Pc(\Xb)$, $h\in[1:H]$,
and cost functions $r_h:\Xb\times\Ab \to [0,1]$, $h\in [1:H-1]$,  $r_H:\Xb\to[0,1]$. We use the notation $[1:H]$ for $\{1,\dots,H\}$. Here, $H$ is the horizon, and $\Pc(\Xb)$ denotes the set of probability distributions on $\Xb$.
Thus, at stage~$h$ and state $x_h\in \Xb$, an action $a\in \Ab$ results in a random transition to a new  state, with probability distribution $\Pb_h(x_h,a)$.
Generally, for each state $x\in \Xb$ and each stage $h\in [1:H]$, we may have a feasible action set $A_h(x) \subset \Ab$, but for  simplicity of exposition we allow all actions;
our results apply to the general case as well. We assume that the spaces $\Xb$ and $\Ab$
are finite but large, so that direct enumeration of all state--action pairs is intractable.

Instead of the standard expected-value objective, we consider a risk-averse model with a Markov risk measure $\rho_{1,H}\big(r_1(x_1,a_1),\dots, r_{H-1}(x_{H-1},a_{H-1}),r_H(x_H)\big)$
used to evaluate the quality of the control policy. Our objective is to minimize it over the space of all feasible policies. Our choice of the minimization formulation is motivated by the convenience of dealing with risk measures and their convexity, but all our results can be easily translated to the maximization setting, with concave risk measures.

Three difficulties are associated with this general setting.  First,  risk measures are nonlinear with respect to the probability measure, and their statistical estimation
is hard; only special approaches for particular cases exist.
To overcome this difficulty, we introduce a new class of dynamic measures of risk that is more amenable to
learning in a setting with a generative model.  Second, for very large state and action spaces, the tabular approach is infeasible. Therefore,
to facilitate the generalization of the risk estimates to unseen state--action pairs, we adopt the  $Q$-value  approximation approach, well-developed in the risk-neutral MDP literature. This gives rise to the
next difficulty resulting from the nonlinear dependence of risk on the probability measure (for a fixed value function) and
the nonlinear dependence on the value function for a fixed probability measure. We
address it by generalizing the class of linear MDPs to a broader class of multipattern MDPs. We propose a risk-averse version of the Least Squares Value Iteration method  and we
prove a high-probability regret bound.

Tabular reinforcement learning for risk-neutral (expected value) problems
was studied by numerous authors, both model-based  \citep{auer2008near,azar2017minimax,dann2017unifying},
and model-free \citep{jin2018q,strehl2009reinforcement}. Generative model settings were discussed, among others, by
\cite{lattimore2012pac,sidford2018near}.
 The linear value function approximation approaches to MDPs have a long history; see
 \citep{bradtke1996linear,sutton1988learning,tsitsiklis1997analysis,Sutton1998,melo2007q}
and many references therein. The concept of a linear (mixture) MDP \citep{bradtke1996linear,melo2007q}, originating from the theory of linear bandits \citep{bubeck2012regret,lattimore2020bandit},  is pertinent in this setting.
It has  recently been used by
\cite{ayoub2020model,yang2020reinforcement,jin2020provably,modi2020sample} to develop complexity bounds for expected value RL algorithms.
The generalization to bilinear models by \cite{du2021bilinear} extends the applicability of this class. The state aggregation approaches of \cite{dong2019provably} are also covered by this model.

The general theory of dynamic risk is discussed by
 \cite{Scandolo:2003,CDK:2006,ADEHK:2007,cheridito2011composition}
 and the references therein. Our approach uses Markov dynamic risk measures developed by  \cite{Ruszczynski2010Markov,shen2013risk,lin2013dynamic,bauerle2022markov}.

 Several works introduce models of risk into reinforcement learning: exponential utility functions \citep{Borkar2001, Borkar2002,basu2008learning,fei2022cascaded} and  
 mean--variance models
\citep{Tamar2012, Prashanth2014}. A few later studies propose heuristic approaches involving specific coherent risk measures, such as
CVaR in the objective or constraints \citep{Chow2014,Chow2015a,ma2018risk}.
Generative model value iteration with coherent measures was analyzed by \cite{yu2018approximate}. Risk-aware Q-learning with Markov risk measures is considered by \cite{huang2017risk}.  All these methods apply to problems with a small number of state--action pairs, which allows exhaustive experimentation.

Value function approximations in the context of  distributionally robust MDPs were considered  by \cite{Tamar2014}.
\cite{Tamar2017} study the policy gradient approach for Markov risk measures and uses it in an actor--critic type algorithm. Both approaches are heuristic. Policy evaluation with a linear architecture and Markov risk measures by a method of temporal differences was analyzed by \cite{kose2021risk}, and asymptotic convergence was proved.  The recent work of \cite{yin2022near} uses variance to control the value iteration procedure.  \cite{fei2020risk,fei2021risk} focus exclusively on the entropic risk measure, but derive theoretically sound and near-optimal regret bounds.
Recently, \cite{lamrisk} proposed a version
of a risk-aware reinforcement learning method with coherent risk measures and function approximation. However, at each iteration, it uses extensive experimentation (double sampling)  to statistically estimate the risk with high accuracy and high probability at each state--action pair.

The objective of this paper is to present a risk-averse $Q$-learning method using linear function approximation, evaluate its regret, and demonstrate its efficacy on two examples.
In Section \ref{s:Markov_risk} we review the theory of Markov risk measures and risk-averse dynamic programming. Special attention is paid to mini-batch risk measures, which allow for unbiased statistical estimation of transition risk. Section \ref{s:multipattern} introduces a new theoretical concept of a multipattern risk-averse MDP, which postulates the representation of transition risk as a combination of
basic risk measures (risk patterns). This modeling assumption allows for the use of a linear function approximation in representing $Q$-factors via features of the state--action pair. In Section~\ref{s:method}
we present the risk-averse $Q$-learning method with function approximation and discuss its convergence mechanism. We also propose an economical version, called ``lazy'' learning, which performs
the expensive value function calculation (minimization of $Q$-factors) at random episodes selected with low probability. The section provides an intuitive explanation of the convergence mechanism and states the main theorem about a high-probability regret bound of 
$\widetilde{\mathcal{O}}\big(d^{3/2}HN^{H-1}\sqrt{K}\big)$, where $d$ is the number of features, $H$ is the horizon, $N\ge2$ is the mini-batch size, and $K$ is the number of episodes. The proof of the theorem, involving several auxiliary lemmas, is provided in the Appendix. Section \ref{s:assignment} demonstrates the operation of the method on a stochastic assignment problem, where the analytical solution is known. In Section \ref{s:bandit},
we present results obtained on a short-horizon multi-armed bandit problem, in which substantial risk is involved.

\section{Markov risk measures and risk-averse dynamic programming}
\label{s:Markov_risk}

\subsection{Dynamic risk measurement}
Suppose the actions are generated by a deterministic policy $\pi$, so that $a_h = \pi_h(x_{1:h})$, where $x_{1:h}=(x_1,\dots,x_h)$ is the history of the process up to stage $h$.
A \emph{dynamic risk measure} evaluates the sequence of the random costs $Z_h=r_h(x_h,a_h)$, $h\in[1:H-1]$, and $Z_H=r_H(x_H)$ in a risk-aware fashion. We denote by $\Zb_h$ the space
of all real functions of the history $x_{1:h}$. We write $\Zb_{h:H}$ for $\Zb_h \times \dots \times \Zb_H$, and $Z_{h:H}$ for $(Z_h,\dots,Z_H)$.
 Because of the need to evaluate future costs at any period,  a dynamic risk measure is a collection of
 \emph{conditional risk measures} $\rho_{h,H}:\Zb_{h:H}\to\Zb_h$, $h\in [1:H]$, for which  we postulate three properties:

 \begin{tightitemize}
 \item[\textbf{Normalization:}] $\rho_{h,H}(0,\dots,0)=0$;
\item[\textbf{Monotonicity:}] If $Z_{h:H} \leq W_{h:H}$, then $\rho_{h,H}(Z_{h:H}) \le \rho_{h,H}(W_{h:H})$;
\item[\textbf{Translation:}] $\rho_{h,H}(Z_{h:H}) = Z_h + \rho_{h,H}(0,Z_{h+1:H})$.
\end{tightitemize}
Here and below, all inequalities are understood component-wise.

 Fundamental to such a nonlinear dynamic cost evaluation is \emph{time consistency},
 discussed in various forms by \cite{CDK:2006,ADEHK:2007,Ruszczynski2010Markov,cheridito2011composition}:
\emph{A dynamic risk measure is time consistent if,  for every $h\in [1:{H-1}]$, the conditions $Z_h=W_h$ and $\rho_{h+1,H}(Z_{h+1:H}) \le \rho_{h+1,H}(W_{h+1:H})$ imply}
$\rho_{h,H}(Z_{h:H}) \le \rho_{h,H}(W_{h:H})$.
As proved by \cite{Ruszczynski2010Markov}, such measures, under the conditions specified above,
must have the recursive form:
\[
\rho_{h,H}(Z_{h:H}) = Z_h + \rho_h\Big(Z_{h+1}+ \rho_{h+1}\big(Z_{h+2}+ \dots + \rho_{H-1}(Z_H)\cdots\big)\Big),
\]
 where each $\rho_h(\cdot)$ is a one-step conditional risk measure. This formula, generalizing the tower property of conditional expectations, is pertinent to our approach.

 \subsection{Markov risk measures}

We want to solve the problem of minimizing $\rho_{1,H}(Z_{1:H})$ over a policy $\pi$ in the
set $\Pi$ of all feasible deterministic history-dependent policies. We choose the minimization
setting to simplify manipulations with convex measures of risk. In order to obtain a dynamic programming formulation, we focus on
\emph{Markov risk measures}. They evaluate the risk-adjusted value of future costs $Z_{h:H}$ in a Markov system with a Markov control policy
$a_h = \pi_h(x_h)$, $h\in [1:H]$,
in such a way that the value of the future cost sequence is a function of the current state:
$\rho_{h,H}(Z_{h:H}) = V_h^\pi(x_h)$, where $V_h^\pi \in \Vb \eqdef \Rb^{|\Xb|}$.
 This, combined with the properties specified above, implies 
 that \emph{transition risk mappings}
 $\sigma_h:\Xb \times  \Pc(\Xb)\times \Vb\to \Rb$, $h\in[1:H-1]$, exist such that the value of each state
 for the Markovian policy $\pi$ can be evaluated by the following procedure:
 \begin{equation}
\label{DP-risk-finite}
\begin{aligned}
V_h^\pi(x) &=  r_h(x,\pi_h(x)) +  \sigma_h\big(x,\Pb_h(x,\pi_h(x)), V^\pi_{h+1}\big),
\quad  x\in \Xb, \quad h\in [1:H-1];\\
V_H^\pi(x) &= r_H(x), \quad  x \in \Xb;
\end{aligned}
\end{equation}
see \citep{Ruszczynski2010Markov,fan2022process} for a detailed derivation.

In such an MDP, a Markovian optimal policy $\pi^\star$ exists, and the corresponding optimal
value functions $V_h^\star(\cdot)$ satisfy the risk-averse dynamic programming equations:
 \begin{equation}
\label{DP-risk-finite-optimal}
\begin{aligned}
V_h^\star(x) &=  \min_{a\in \Ab} \Big\{r_h(x,a) +  \sigma_h\big(x,\Pb_h(x,a), V_{h+1}^\star\big)\Big\},\quad
  x\in \Xb, \quad h\in [1:H-1];\\
V_H^\star(x) &= r_H(x), \quad  x \in \Xb.
\end{aligned}
\end{equation}
The policy $\pi^\star$ is given by the minimizers in \eqref{DP-risk-finite-optimal}.

The $Q$-factors are defined for a risk-averse model in a way similar to the risk-neutral case:
\begin{align*}
Q_h^\pi(x,a)&=r_h(x,a)+\sigma_h\big(x,\Pb_h(x,a),V_{h+1}^\pi\big),\\
Q_h^\star(x,a)&=r_h(x,a)+\sigma_h\big(x,\Pb_h(x,a),V_{h+1}^\star\big),
\end{align*}
and therefore,
$V_h^\pi(x)=Q_h^\pi(x,\pi_h(x))$ and
$V_h^\star(x)=\min_{a\in\Ab}Q_h^\star(x,a)$, $h\in[1:H-1]$. The $Q$-factors are elements of the space $\Qb \eqdef \Rb^{|\Xb\times \Ab|}$.

If the stage-wise costs are random, denoted here by $R_h(x,a)$, the dynamic programming equations are modified accordingly. Denote by $\Pi_h(x,a)$ the joint distribution
of the cost $R(x,a)$ and the next state $x' \sim \Pb(x,a)$ in the product space $\Rb \times \Xb$. Then, after transformations detailed in \cite[\S 8.1.5]{dentcheva2024risk}, we obtain the equations:
 \begin{equation}
\label{DP-risk-finite-optimal-random}
\begin{aligned}
V_h^\star(x) &=  \min_{a\in \Ab} \Big\{\sigma_h\big(x,\Pi_h(x,a), R_h +  V_{h+1}^\star\big)\Big\},\quad
  x\in \Xb, \quad h\in [1:H-1];\\
V_H^\star(x) &= r_H(x), \quad  x \in \Xb.
\end{aligned}
\end{equation}
In a similar way, the $Q$-factors are:
\begin{equation}
\label{Q-factors-random}
\begin{aligned}
Q_h^\pi(x,a)&=\sigma_h\big(x,\Pi_h(x,a),R_h + V_{h+1}^\pi\big),\\
Q_h^\star(x,a)&=\sigma_h\big(x,\Pi_h(x,a),R_h + V_{h+1}^\star\big).
\end{aligned}
\end{equation}

In the rest of this section, for simplicity, we skip the subscript $h$ of the transition risk $\sigma_h(\cdot,\cdot,\cdot)$ and of the kernel  $\Pb_h(\cdot,\cdot)$.

A simple example of a transition risk  mapping is the conditional expected value:
\begin{equation}
\label{E-form}
\sigma(x,\Pb(x),V) = \Pb(x) V = \sum_{y\in \Xb} \Pb(y|x) V(y).
\end{equation}
Its use reduces \eqref{DP-risk-finite}--\eqref{DP-risk-finite-optimal} to the standard expected-value dynamic programming.

In our analysis, we are interested in transition risk  mappings depending on the transition probabilities and the value function in a nonlinear way. An
example is the \emph{mean--semideviation model}
\begin{equation}
\label{msd-form}
\sigma(x,\Pb(x),V) = \Pb(x) V
+ \kappa \sum_{y\in \Xb} \Pb(y|x)\max\big( 0, V(y) -  \Pb(x) V\big), \quad \kappa\in [0,1].
\end{equation}
Another example is the \emph{Average Value at Risk} (AVaR):
\begin{equation}
\label{avar}
\sigma(x,\Pb(x),V) = \min_{\eta\in \Rb}\Big\{ \eta + \frac{1}{\alpha} \Pb(x) \big[ (V - \eta)_+\big]\Big\}, \quad \alpha\in (0,1].
\end{equation}
Yet another example, rarely used in risk measure theory because of its conservative nature, but relevant for us, is the \emph{worst-case} mapping:
\begin{equation}
\label{sup-form}
\sigma(x,\Pb(x),V) =  \inf \big\{ b\in \Rb: \Pb\big[V(y)\le b \,\big|\, x\big] = 1 \big\} .
\end{equation}
All the examples above, as functionals of the value function $V(\cdot)$, with a fixed $P=\Pb(x)$, satisfy the axioms of a coherent measure of risk \citep{ArtznerDelbaenEberEtAl1999}:
\begin{tightitemize}
\item[\textbf{Convexity:}] $\sigma(x,P,\lambda V + (1-\lambda) W) \leq  \lambda \sigma(x,P,V) + (1-\lambda) \sigma(x,P,W)$,  $\forall\, \lambda \in [0,1]$;
\item[\textbf{Monotonicity:}] If $V \leq W$ then $\sigma(x,P,V) \leq \sigma(x,P,W)$;
\item[\textbf{Translation:}] $\sigma(x,P,V + \beta\1) = \sigma(x,P,V) + \beta$, for all $\beta \in \Rb$ (here, $\1$ is the vector of ones);
\item[\textbf{Positive homogeneity:}] $\sigma(x,P,\alpha V) = \alpha \sigma(x,P,V)$, for all \mbox{$\alpha \geq 0$}.
\end{tightitemize}
For such transition risk mappings, the dual representation is valid:
\begin{equation}
\label{dual}
\sigma(x,\Pb(x),V) = \max \Big\{  \Qb V: {\frac{\D \Qb}{\D\Pb(x)} \in \Ac(x)}\Big\},
\end{equation}
with some convex and closed set $\Ac(x)$ of densities; see \cite{RuszczynskiShapiro2006a} and the comprehensive exposition in 
\cite[Ch. 2]{dentcheva2024risk}.

It should be stressed that convex combinations of coherent measures of risk are coherent. Specifically, we may mix the Average Values at Risk with different levels $\alpha$ with the expected value and construct the so-called \emph{spectral} measures of risk. We may also include the worst-case mapping in the mix.

\subsection{Statistical estimation of transition risk}

A serious difficulty associated with the risk-averse control model described above is the problem of statistical estimation of the transition risk $\sigma(x,\Pb(x),V)$, due to its nonlinear dependence on $\Pb(x)$. The literature provides results for special cases of measures of risk, all requiring repeated simulation from the conditional distribution $\Pb(x)$; see \citep{dentcheva2024risk,shapiro2021lectures}.
In the context of optimization and control of risk, the Average Value at Risk \eqref{avar} (or more generally certainty equivalent risk measures) is an exception, due to its definition as the minimum of an expectation. We can consider \emph{augmented controls} $\bar{a}_h = \{a_h,\eta_h\} \in \Ab \times \Rb$ and
define the $Q$-factors as
 \[
Q_h(x,\bar{a}) =  r_h(x,a) +  \eta  +  \frac{1}{\alpha} \Pb(x,a)\big[ 
  (V_{h+1}^\star - \eta)_+\big],
\]
with
$V_{h+1}^\star(x') = \min_{{a\in \Ab}, \, {\eta\in \Rb}}Q_{h+1}(x',(a,\eta))$.
This allows us to observe a transition from a state $x$ with control $a$ to a random state $Y$ and use the random quantity
\[
\widetilde{Q}_h(x,\bar{a}) =  r_h(x,a) +  \eta  +
 \frac{1}{\alpha} (V_{h+1}^\star(Y) - \eta)_+
\]
as an unbiased observation of $Q_h(x,\bar{a})$. A similar construction can be applied to spectral risk measures involving mixtures of the expected value  and several AVaRs.

Our focus  is on a new approach involving a special class of transition risk mappings, which we call
\emph{mini-batch mappings}. Below, we adapt to our case the construction of \citep{dentcheva2023mini}. For
a sample $Y^{1:N}=(Y^1,\dots,Y^N)$ with $N$ independent observations of the next state distributed according to $\Pb(x)$ in $\Xb$, we obtain a random empirical measure
$\Pb^{(N)}(x)=\frac{1}{N}\sum_{j=1}^N\delta_{Y^j}$ (each $\delta_{Y^j}$ is the Dirac measure with unit mass at $Y^j$). Then, for a ``base'' coherent
{ law-invariant} transition risk mapping $\sigma^{\text{b}}(x,\Pb(x),V)$, we consider the random mapping $\sigma^{\text{b}}(x,\Pb^{(N)}(x),V)$. Its expected value is a new coherent transition risk mapping
\begin{equation}
\label{mini-batch}
\sigma^{(N)}(x,\Pb(x),V)=\Eb_{Y^{1:N}\sim (\Pb(x))^N}\big[\sigma^{\text{b}}(x,\Pb^{(N)}(x),V)\big].
\end{equation}
An essential virtue of this construction is that a law invariant $\sigma^{\text{b}}(x,\Pb^{(N)}(x),V)$ is a function of the $N$ values $V(Y^j)$, $j=1,\dots,N$, which can be computed from the batch $Y^{1:N}$.

That $\sigma^{(N)}(x,\Pb(x), \cdot)$ is coherent follows directly from the four axioms. Evidently, if $\sigma^{\text{b}}(x,\Pb(x),\cdot)$
is the conditional expectation \eqref{E-form}, the mini-batch mapping \eqref{mini-batch} is the same, due to the tower property. However, in the examples
\eqref{msd-form} and \eqref{sup-form}, the formula \eqref{mini-batch} defines  new transition risk mappings. Thanks to the small batch involved and the use of the
expectation, their statistical estimation and learning become tractable, as we shall demonstrate below.
Particularly convenient is the mini-batch counterpart of the worst-case mapping \eqref{sup-form}:
\begin{equation}
\label{batch-sup}
\sigma^{(N)}(x,\Pb(x),V) = \Eb_{Y^{1:N}\sim (\Pb(x))^N}\big[ \textstyle{\max_{1 \le j \le N}} V(Y^j) \big],
\end{equation}
the MAXVAR measure of \cite{cherny2009new}.
Even for $N=2$ it defines a nontrivial yet tractable model of risk aversion. We stress again that we do not consider $N\to \infty$, but rather work with a fixed (and possibly very small) $N$.

Combining the measure \eqref{batch-sup} with the conditional expectation, we obtain the $Q$-factors,
\[
Q_h(x,a) = r_h(x,a) + \Eb_{Y^{1:N}\sim (\Pb_h(x,a))^N} \Big[ (1-\varkappa) \frac{1}{N} \sum_{j=1}^N V_{h+1}^\star(Y^j)
+ \varkappa {\max_{1 \le j \le N}} V_{h+1}^\star(Y^j)\Big],
\]
and their unbiased statistical estimates:
\begin{equation}
\label{Psi-mix}
\widetilde{Q}_h(x,a) = r_h(x,a) +  (1-\varkappa) \frac{1}{N} \sum_{j=1}^N V_{h+1}^\star(Y^j)
+ \varkappa {\max_{1 \le j \le N}} V_{h+1}^\star(Y^j).
\end{equation}
If the stage-wise costs are random, the construction is similar; see \eqref{DP-risk-finite-optimal-random}--\eqref{Q-factors-random}.
We observe a sample $(R,Y)^{1:N}$ of the  cost--state pairs resulting from state $x$. These allow
us to construct an empirical measure $\Pi^{(N)}(x)=\frac{1}{N}\sum_{j=1}^N\delta_{(R,Y)^j}$, where each $\delta_{(R,Y)^j}$ is the Dirac measure with unit mass at $(R,Y)^j$. Then, for a base coherent and law invariant
transition risk mapping $\sigma^{\text{b}}(x,\Pi(x),R+V)$, we consider the random mapping $\sigma^{\text{b}}(x,\Pi^{(N)}(x),R+V)$. Its expected value is a new coherent transition risk mapping
\begin{equation}
\label{mini-batch-random}
\sigma^{(N)}(x,\Pi(x),R+V)=\Eb_{(R,Y)^{1:N}\sim (\Pi(x))^N}\big[\sigma^{\text{b}}(x,\Pi^{(N)}(x),R+V)\big].
\end{equation}

\section{Multipattern risk-averse Markov decision processes}
\label{s:multipattern}

The dynamic programming equations \eqref{DP-risk-finite-optimal} can be solved by a tabular approach only for small state spaces. To treat practically relevant
problems, we adapt  feature-based approximation to our setting. Let $\Sb_d$ be the unit simplex in $\Rb^d$.
We assume that a \emph{feature map} $\varphi:\Xb\times\Ab\rightarrow \Sb_d$ is available such that
the stage-wise cost and the transition risk can be deduced from the features of the state--action pair. This requires a modeling assumption
that will allow us to learn from the observations at other state--action pairs.

\begin{definition}
\label{d:multipattern}
{\rm We say that the model is a \emph{Multipattern Risk-Averse MDP} if for any $h\in [1:H]$, there exist $d$ unknown
coherent measures of risk $\rho_{hj}:\Vb \to \Rb$, $j=1,\dots,d$, 
 and an unknown vector $\theta_h \in \mathbb{R}^d$,  such that for any $(x,a)\in\Xb\times\Ab$
 and all $V_{h+1}\in \Vb$, with $\rho_h = \big\{\rho_{hj}\big\}_{j=1}^d$, we have
\begin{align}
\sigma_h(x, \Pb_h(x,a), V_{h+1}) &= \sum_{j=1}^d\varphi_j(x,a)\, \rho_{hj}\big(V_{h+1}\big) = \varphi(x,a)^{\top}\rho_{h}\big(V_{h+1}\big), \label{multipattern1}\\
r_h(x,a) &= \varphi(x,a)^{\top} \theta_h.\label{multipattern2}
\end{align}
We also assume that $\|\theta_h\|\le \sqrt{d}$. By a coherent measure of risk, we understand a function on $\Vb$ that satisfies the four axioms of convexity, monotonicity, translation, and positive homogeneity.}
\end{definition}

We call the measures $\rho_{hj}(\cdot)$, $j=1,\dots,d$, the \emph{risk evaluation patterns}. As they are coherent measures of risk, \eqref{multipattern1} is consistent with the coherence
of $\sigma_h(x,\Pb_h(x,a),\cdot)$ (recall that $\varphi(x,a)\in \Sb_d$). We may also work with different feature maps $\varphi_h:\Xb\times\Ab\rightarrow \Sb_{d_h}$ for each stage $h \in [1:H]$, adapting Definition \ref{d:multipattern} accordingly. All our considerations
remain valid in this case; only the notation becomes more involved. {We also note that for coherent patterns the bound $\|\rho_h(\1)\|\le\sqrt{d}$ is automatic: positive homogeneity gives $\rho_{hj}(0)=0$, and the translation property then forces $\rho_{hj}(\1)=\rho_{hj}(0)+1=1$.}

In the special case of linear patterns,
with $\rho_h(V_{h+1}) = \int_\Xb V_{h+1}(x')\,\mu_h(\D x')$, where $\mu_h=(\mu_h^{(1)},\dots,\mu_h^{(d)})$ is a vector of $d$ unknown { probability} measures, the multipattern MDP reduces to a linear MDP studied in the literature.

In the tabular approach, we consider all possible state--action pairs $(x,a)_j$, $j=1,\dots,d$. Then
\eqref{multipattern1} is satisfied with  $\varphi_j(x,a) = 1$ if $(x,a)=(x_j,a_j)$, and 0 otherwise, $j=1,\dots,d$; and $\rho_{hj}(\cdot)\equiv \sigma_h(x_j,\Pb_h(x_j,a_j),\,\cdot\,)$ for all $j=1,\dots,d$. In this case, we have as many patterns as there are state--action pairs
and the model is exact.

Definition \ref{d:multipattern} implies that for each $h$ and any policy $\pi$ we have
\begin{equation}
\label{Q-w-exact}
Q_h^\pi(x,a)   =  \varphi(x,a)^\top w_h^\pi,
\end{equation}
with some weight vector $w_h^\pi$. Our objective is to learn the weights corresponding to the optimal policy.

\section{The method}
\label{s:method}

We present the method for the case of a mini-batch transition risk mapping. To simplify the notation, we ignore the direct dependence of the transition risk mappings on the state,
and write $\sigma_h(\Pb_h(x,a), V_{h+1})$ instead of  $\sigma_h(x,\Pb_h(x,a), V_{h+1})$, but all our considerations are valid for the general structure as well. {From now on, the transition risk mapping of the model is the mini-batch mapping
itself: the dynamic programming equations
\eqref{DP-risk-finite}--\eqref{DP-risk-finite-optimal} and Definition
\ref{d:multipattern} are understood with $\sigma_h=\sigma^{(N)}_h$, so that the
multipattern structure \eqref{multipattern1} is postulated for $\sigma^{(N)}_h$. This is
consistent, because $\sigma^{(N)}_h$, being an expectation of the coherent base mapping
evaluated at the random empirical measure, is itself a coherent transition risk
mapping.}

\subsection{The risk-averse $Q$-learning algorithm}

We denote by $Y_{h+1}^k$ the $N$-tuple of conditionally independent observations $\{Y_{h+1}^{k,1},\dots,Y_{h+1}^{k,N} \}$ of a state that may result from
applying the action $a_h^k$ at the state $x_h^k$ in episode $k$ at stage $h$. For any function $V_{h+1}(\cdot)\in \Vb$, we
define the \emph{risk aggregator} $\varPsi_h\big(\big\{ V_{h+1}( Y_{h+1}^{\tau,j})\big\}_{j\in [1:N]}\big) = \sigma^{\text{b}}_h\big(\Pb_h^{(N)}(x_h^\tau),V_{h+1}\big)$, where $\Pb_h^{(N)}(x_h^\tau)$
is the empirical measure associated with the mini-batch $Y_{h+1}^\tau$.
{To describe the information structure, for $\tau\in[1:K]$ and $h\in[1:H]$ we denote by $\Fc_h^\tau$ the $\sigma$-algebra generated by all observations of the episodes $1,\dots,\tau-1$ (states, actions, and mini-batches) and by the observations of episode $\tau$ up to the choice of the action at stage $h$:
\begin{equation}
\label{filtration}
\Fc_h^\tau=\sigma\Big(\big\{x_{h'}^\nu,a_{h'}^\nu\big\}_{\nu<\tau,\,h'\in[1:H]},\ \big\{Y_{h'+1}^\nu\big\}_{\nu<\tau,\,h'\in[1:H-1]},\ \big\{x_{h'}^\tau,a_{h'}^\tau\big\}_{h'\in[1:h]},\ \big\{Y_{h'+1}^\tau\big\}_{h'\in[1:h-1]}\Big).
\end{equation}
These $\sigma$-algebras increase in the lexicographic order of $(\tau,h)$:
$\Fc_1^\tau\subset\Fc_2^\tau\subset\dots\subset\Fc_H^\tau\subset\Fc_1^{\tau+1}$.
The feature vector $\varphi(x_h^\tau,a_h^\tau)$ is $\Fc_h^\tau$-measurable.
The weights, the Gram matrices, and the functions $Q_h^\tau(\Cdot,\Cdot)$ and $V_h^\tau(\Cdot)$ computed at the beginning of episode $\tau$ depend only on the data of the episodes $1,\dots,\tau-1$; they are thus $\Fc_1^\tau$-measurable.
Conditionally on $\Fc_h^\tau$, the mini-batch $Y_{h+1}^\tau$ consists of $N$ independent observations from $\Pb_h(x_h^\tau,a_h^\tau)$; it is $\Fc_{h+1}^\tau$-measurable, and hence also $\Fc_h^{\tau+1}$-measurable.
By \eqref{mini-batch}, for every $V_{h+1}\in\Vb$ that is $\Fc_h^\tau$-measurable (in particular, for every fixed $V_{h+1}\in\Vb$),
\begin{equation}
\label{mini-batch-estimate}
\sigma^{(N)}_h(\Pb_h(x_h^\tau,a_h^\tau), V_{h+1})=\Eb\Big[\varPsi_h\big(\big\{ V_{h+1}( Y_{h+1}^{\tau,j})\big\}_{j\in [1:N]}\big)\,\Big|\,\Fc_h^\tau\Big].
\end{equation}
We stress that the functions $V_{h+1}^k$, $k>\tau$, used in the regression step of Algorithm \ref{a:basic} depend on $Y_{h+1}^\tau$ and are not $\Fc_h^\tau$-measurable.}

Algorithm \ref{a:basic}  presents an adaptation of the $Q$-learning method from \cite{jin2020provably} to the risk-averse case discussed here.

\begin{algorithm}
\SetAlgoLined
\For{episode $k\leftarrow 1$ \KwTo $K$}{
\emph{Receive the initial state $x_1^k$}\;
\For{stage $h\leftarrow H$ \KwTo $1$}{
$\qquad\quad\ \Lambda_h^k\leftarrow\textstyle{\sum_{\tau=1}^{k-1}}\varphi(x_h^\tau,a_h^\tau)\,\varphi(x_h^\tau,a_h^\tau)^{\top}+\lambda\cdot\mathbb{I};\ (\lambda>0)$;\par
\eIf{$h=H$}{
$\qquad\quad\  w_H\leftarrow(\Lambda_H^k)^{-1}\textstyle{\sum_{\tau=1}^{k-1}}\varphi(x_H^\tau,a_H^\tau)\,r_H(x_H^\tau);
$\par
}{
$\qquad\quad\  w_h\leftarrow(\Lambda_h^k)^{-1}\textstyle{\sum_{\tau=1}^{k-1}}\varphi(x_h^\tau,a_h^\tau)\,\Big[r_h(x_h^\tau,a_h^\tau)
+ { \varPsi_h\Big(\big\{ \min_a Q_{h+1}^k( Y_{h+1}^{\tau,j},a)\big\}_{j\in [1:N]}\Big)}\Big];
$\par
}
$\qquad\quad\  Q_h^k(\cdot,\cdot)\leftarrow{\textup{Proj}_{[0,H]}
\Big(w^{\top}_h\varphi(\cdot,\cdot)-\beta\,[\varphi(\cdot,\cdot)^{\top}(\Lambda_h^k)^{-1}\varphi(\cdot,\cdot)]^{1/2}\Big)}$.
}
\For{stage $h\leftarrow 1$ \KwTo ${H}$}{
$\qquad\quad\ $ $a_h^k\leftarrow \argmin_{a\in\mathbb{A}}Q_h^k(x_h^k,a)$;\par
\If{{$h\le H-1$}}{
\  Observe an $N$-sized sample $Y_{h+1}^k$ from $\Pb_h(x_h^k,a_h^k)$ and
            define $x_{h+1}^k$ as a randomly chosen state in $Y_{h+1}^k$ 
 {(uniformly, independently of the batch values)}.\par
}
}
}
\caption{Feature-Based Risk-Averse Mini-Batch $Q$-Learning}
\label{a:basic}
\end{algorithm}
\vspace{1ex}

If the stage-wise costs are random, the construction is similar and follows \eqref{mini-batch-random}. We denote by $(R_{h},Y_{h+1})^k$ the $N$-tuple of conditionally independent observations 
$(R_h,Y_{h+1})^{k,j}$, $j=1,\dots,N$, of the cost and the state that may result from
applying the action $a_h^k$ at the state $x_h^k$ in episode $k$ at stage $h$. For any function $V_{h+1}(\cdot)\in \Vb$, we
define $\varPsi_h\big(\big\{ R_h^{\tau,j} + V_{h+1}( Y_{h+1}^{\tau,j})\big\}_{j\in [1:N]}\big) = \sigma^{\text{b}}_h\big(\Pi_h^{(N)}(x_h^\tau),R_h + V_{h+1}\big)$, where $\Pi_h^{(N)}(x_h^\tau)$
is the empirical measure associated with the mini-batch $(R_h,Y_{h+1})^\tau$. In Algorithm \ref{a:basic}, the regression step is modified accordingly:
\[
 w_h\leftarrow(\Lambda_h^k)^{-1}{\sum_{\tau=1}^{k-1}}\varphi(x_h^\tau,a_h^\tau)\,
\varPsi_h\Big(\big\{R_h^{\tau,j} + \min_a Q_{h+1}^k( Y_{h+1}^{\tau,j},a)\big\}_{j\in [1:N]}\Big).
\]
{Three comments on the algorithm are in order. First, the projection in the definition of $Q_h^k$ guarantees 
$0\le V_{h+1}^k=\min_a Q_{h+1}^k(\cdot,a)\le H$, {on every sample path}. This is
essential for the covering argument behind Lemma \ref{lma53} and the
martingale bound in the proof of Theorem \ref{t:main}.
Second, the action $a_H^k$ selected at the terminal stage is payoff-irrelevant --- the
terminal cost depends on the state only, and no transition follows --- but it supplies
the design point $\varphi(x_H^k,a_H^k)$ entering the Gram matrices $\Lambda_H^{k'}$ of
later episodes $k'>k$ and the bonus term of the regret recursion
\eqref{delta-recursive}. Accordingly, we adopt the conventions
$V_{H+1}^\pi\equiv V_{H+1}^k\equiv0$, $\sigma_H(x,P,0)=0$, and
$Q_H^\pi(x,a)=Q_H^\star(x,a)=r_H(x)$ for all $a\in\Ab$. Third, the uniform, value-independent selection of $x_{h+1}^k$ from
the batch guarantees that, conditionally on the past, $x_{h+1}^k\sim\Pb_h(x_h^k,a_h^k)$
and
$\Eb\big[W(x_{h+1}^k)\,\big|\,\Fc_h^k\big]
=\frac1N\sum_{j=1}^N\Eb\big[W(Y_{h+1}^{k,j})\,\big|\,\Fc_h^k\big]$
for every function $W$ determined by $\Fc_h^k$.}

\subsection{Economical policy evaluation}

The main computational cost in Algorithm \ref{a:basic} is the  policy evaluation step in the backward pass, which involves
the calculation of the new policy values: $V_{h+1}^k\big(Y_{h+1}^{\tau,j}\big) = \min_a Q_{h+1}^k( Y_{h+1}^{\tau,j},a)$ for each episode~$k$, every past episode $\tau$, and all batch states $j$.
They are used in the calculation of the conditional risk estimate $\varPsi_h\big(\big\{ V_{h+1}^k\big(Y_{h+1}^{\tau,j}\big)\big\}_{j\in [1:N]}\big)$. Altogether,
$\Oc(K^2HN)$ optimization
problems must be solved, each involving $|\Ab|$ possible action values. However, the analysis outlined in the next subsection and detailed in Appendix~\ref{s:regret}
uses only the following properties of the risk estimates in the regression step. For every
$Q$-function $Q_{h+1}\in\Qb$ and $V_{h+1}(\cdot)=\min_a Q_{h+1}(\cdot,a)$, the estimate
$\widetilde{\varPsi}_{h+1}^{\tau}(Q_{h+1})$ is determined by $Q_{h+1}$ and the data in
$\Fc_h^{\tau+1}$, takes values in $[0,H]$, is $1$-Lipschitz in $Q_{h+1}$ with respect to
$\|\cdot\|_\infty$, and satisfies
\begin{equation}
\label{mini-batch-estimate-2}
\sigma^{(N)}_h(\Pb_h(x_h^\tau,a_h^\tau), V_{h+1})
=\Eb\big[\widetilde{\varPsi}_{h+1}^{\tau}(Q_{h+1})\,\big|\,\Fc_h^\tau\big].
\end{equation}
The following economical (``lazy'') policy evaluation is a heuristic approximation of such an
estimator.

Together with each batch $Y_{h+1}^\tau$, we store the actions $\tilde{a}_{h+1}^{\tau,j}$ that were best for
$Y_{h+1}^{\tau,j}$, $j=1,\dots,N$, at the latest episode $\ell<k$ at which the exact calculation of $\min_a Q_{h+1}^\ell\big(Y_{h+1}^{\tau,j},a\big)$, $j=1,\dots,N$, was performed.

At episode $k$ and stage $h$ of the backward pass, for every $\tau\in [1:k-1]$, with probability $p^{\text{renew}}$ we follow Algorithm \ref{a:basic}: we solve all $N$ optimization problems, calculate the quantities
\begin{equation}
\label{Case-1}
 \widetilde{V}_{h+1}^{\tau,j} = {V}_{h+1}^k(Y_{h+1}^{\tau,j}\big) = \min_a  Q_{h+1}^k\big(Y_{h+1}^{\tau,j},a\big),\quad j=1,\dots,N,
\end{equation}
and update all $\tilde{a}_{h+1}^{\tau,j}$, $j=1,\dots,N$, to the new solutions.

With probability $1- p^{\text{renew}}$, we keep $\tilde{a}_{h+1}^{\tau,j}$, $j=1,\dots,N$, unchanged and set
\begin{equation}
\label{Case-2}
\widetilde{V}_{h+1}^{\tau,j} =   Q_{h+1}^k\big(Y_{h+1}^{\tau,j},\tilde{a}_{h+1}^{\tau,j}\big) - \Delta_{h+1}^{k,\tau},\quad j=1,\dots,N.
\end{equation}
The bias correction $\Delta_{h+1}^{k,\tau}$ can be estimated as an average of improvements at earlier batches $\nu < \tau$.
{ Let
\[
 \Nc_{h+1}^{k,\tau} = \big\{\nu<\tau: \text{ exact minimization of $ Q_{h+1}^k\big(Y_{h+1}^{\nu,j},a\big)$ was performed for } j=1,\dots,N \big\}.
\]
For $\nu\in\Nc_{h+1}^{k,\tau}$, let $\hat{a}_{h+1}^{\nu,j}$ denote the actions stored for
the batch $Y_{h+1}^{\nu}$ \emph{before} their renewal at episode $k$.} Then
\[
\Delta_{h+1}^{k,\tau}  = \frac{1}{{N}\big|\Nc_{h+1}^{k,\tau}\big|}
\sum_{\nu \in  \Nc_{h+1}^{k,\tau}} \sum_{j=1}^{N}\Big[ Q_{h+1}^k\big(Y_{h+1}^{\nu,j},\hat{a}_{h+1}^{\nu,j}\big)
- \min_a  Q_{h+1}^k\big(Y_{h+1}^{\nu,j},a\big)\Big].
\]
{ When a batch is used for the first time ($k=\tau+1$), the minimization
\eqref{Case-1} is always performed. Consequently, stored actions are available
for every batch at which \eqref{Case-2} is applied, and the actions
$\hat{a}_{h+1}^{\nu,j}$ stored before renewal are available for every
$\nu\in\Nc_{h+1}^{k,\tau}$. If $\Nc_{h+1}^{k,\tau}$ is empty, we use
$\Delta_{h+1}^{k,\tau}=0$.} Finally, we calculate the estimate
\[
\widetilde{\varPsi}_{h+1}^{\tau}(Q_{h+1}^k) = \varPsi_h\Big(\big\{ \widetilde{V}_{h+1}^{\tau,j}\big\}_{j\in [1:N]}\Big).
\]
As the bias correction term is the same for all batch observations,
due to the translation property of $\varPsi_h(\cdot)$,  we have
\[
\widetilde{\varPsi}_{h+1}^{\tau}(Q_{h+1}^k) = \varPsi_h\Big( \Big\{ Q_{h+1}\big(Y_{h+1}^{\tau,j},\tilde{a}_{h+1}^{\tau,j}\big)\Big\}_{j=1,\dots,N}\Big)
- \Delta_{h+1}^{k,\tau}.
\]
Our experience indicates that as the method progresses, the most recently updated actions $\tilde{a}_{h+1}^{\tau,j}$, $j=1,\dots,N$, are optimal or close to optimal, and the bias vanishes.
for a fixed $Q_{h+1}$ the stored actions are in general suboptimal, so that
$Q_{h+1}(Y_{h+1}^{\tau,j},\tilde a_{h+1}^{\tau,j})\ge V_{h+1}(Y_{h+1}^{\tau,j})$ and the estimate
is biased upward (the correction $\Delta_{h+1}^{k,\tau}$ in (21) reduces this bias); moreover,
the stored actions depend on data observed after episode $\tau$. Its merits are assessed
experimentally in Section~\ref{s:assignment}.

This modification is relevant in the expected-value case as well; it substantially reduces the computational effort if $p^{\text{renew}}$ is small.

\subsection{The convergence mechanism}

Before proceeding to the detailed convergence analysis, we provide an intuitive explanation of our ideas.
We focus on the deterministic-cost case, which is more transparent. The random cost case can be treated in the same way, but with more complicated notation.

The vector $w_h$ calculated in the algorithm is the solution to the following ridge regression problem:
\begin{equation}
\notag
w_h\leftarrow\argmin_{w}\sum_{\tau=1}^{k-1}\Big[r_h(x_h^\tau,a_h^\tau)
+\varPsi_h\big(\big\{ V_{h+1}^k( Y_{h+1}^{\tau,j})\big\}_{j\in [1:N]}\big)
-w^{\top}\varphi(x_h^\tau,a_h^\tau)\Big]^2+\lambda\|w\|^2.
\end{equation}
If we ignore the $\beta$-term in the approximation of $Q_h^k(\cdot,\cdot)$, we have
\begin{equation}
\label{Q-approx}
\begin{aligned}
    Q_h^k(x,a) &\approx \varphi(x,a)^Tw_h\\
    &= \varphi(x,a)^{\top}(\Lambda_h^k)^{-1}\sum_{\tau=1}^{k-1}\varphi(x_h^\tau,a_h^\tau)\big[r_h(x_h^\tau,a_h^\tau)+
    \varPsi_h\big(\big\{ V_{h+1}^k( Y_{h+1}^{\tau,j})\big\}_{j\in [1:N]}\big)
\big],
\end{aligned}
\end{equation}
where $V_{h+1}^k\big(Y_{h+1}^{\tau,j}\big)=\min_a Q_{h+1}^k\big(Y_{h+1}^{\tau,j},a\big)$.
Our idea is that the part of \eqref{Q-approx} involving the terms
\begin{equation}
\notag
\hat{\sigma}_h^k(\Pb_h(x,a), V_{h+1}^k)=\varphi(x,a)^{\top}(\Lambda_h^k)^{-1}\sum_{\tau=1}^{k-1}\varphi(x_h^\tau,a_h^\tau)\,\varPsi_h\big(\big\{ V_{h+1}^k( Y_{h+1}^{\tau,j})\big\}_{j\in [1:N]}\big)
\end{equation}
is close to the mini-batch  transition risk mapping $\sigma^{(N)}_h(\Pb_h(x,a), V_{h+1}^k)$, as a function of $(x,a)$.
First, we establish the approximation to
\begin{equation}
\notag
\bar{\sigma}_h^k(\Pb_h(x,a), V_{h+1}^k)=\varphi(x,a)^{\top}(\Lambda_h^k)^{-1}\sum_{\tau=1}^{k-1}\varphi(x_h^\tau,a_h^\tau)\,\sigma^{(N)}_h( \Pb_h(x_h^\tau,a_h^\tau), V_{h+1}^k).
\end{equation}
We have
\begin{multline*}
\hat{\sigma}_h^k(\Pb_h(x,a), V_{h+1}^k)-\bar{\sigma}_h^k(\Pb_h(x,a), V_{h+1}^k)\\
=\varphi(x,a)^{\top}(\Lambda_h^k)^{-1}\sum_{\tau=1}^{k-1}\varphi(x_h^\tau,a_h^\tau)
\Big[ \varPsi_h\big(\big\{ V_{h+1}^k( Y_{h+1}^{\tau,j})\big\}_{j\in [1:N]}\big)-\sigma^{(N)}_h( \Pb_h(x_h^\tau,a_h^\tau), V_{h+1}^k)\Big] .
\end{multline*}
Due to \eqref{mini-batch-estimate},  we can bound the difference above by a concentration inequality for self-normalized martingales.
The same argument applies to any estimator with the properties listed in Section 4.2; see the end of the proof of Lemma 2.

In the next step, we use the multipattern risk-averse MDP structure to show that for all $(x,a)$ we have the approximation
 $\bar{\sigma}_h^k(\Pb_h(x,a), V_{h+1}^k)\approx\sigma^{(N)}_h(\Pb_h(x,a), V_{h+1}^k)$.
{By \eqref{multipattern1}, $\sigma^{(N)}_h(\Pb_h(x_h^\tau,a_h^\tau), V_{h+1}^k)=\varphi(x_h^\tau,a_h^\tau)^{\top}\rho_h\big(V_{h+1}^k\big)$,
and $\sum_{\tau=1}^{k-1}\varphi(x_h^\tau,a_h^\tau)\,\varphi(x_h^\tau,a_h^\tau)^{\top}=\Lambda_h^k-\lambda I$. Therefore,}
\begin{equation}
\notag
\begin{aligned}
\bar{\sigma}_h^k(\Pb_h(x,a), V_{h+1}^k)
&=\varphi(x,a)^{\top}(\Lambda_h^k)^{-1}\sum_{\tau=1}^{k-1}\varphi(x_h^\tau,a_h^\tau)\,\varphi(x_h^\tau,a_h^\tau)^{\top}\rho_{h}\big(V_{h+1}^k\big)\\
&{=\varphi(x,a)^{\top}\rho_{h}\big(V_{h+1}^k\big)-\lambda\,\varphi(x,a)^{\top}(\Lambda_h^k)^{-1}\rho_{h}\big(V_{h+1}^k\big)}\\
&\approx \varphi(x,a)^{\top}\rho_{h}\big(V_{h+1}^k\big) =\sigma^{(N)}_h(\Pb_h(x,a), V_{h+1}^k).
\end{aligned}
\end{equation}
In a similar way, by \eqref{multipattern2},
\begin{multline*}
\varphi(x,a)^{\top}(\Lambda_h^k)^{-1}\sum_{\tau=1}^{k-1}\varphi(x_h^\tau,a_h^\tau)\,r_h(x_h^\tau,a_h^\tau) \\
=
\varphi(x,a)^{\top}(\Lambda_h^k)^{-1}\sum_{\tau=1}^{k-1}\varphi(x_h^\tau,a_h^\tau)\,\varphi(x_h^\tau,a_h^\tau)^{\top}\theta_h
\approx \varphi(x,a)^{\top}\theta_h = {r_h(x,a)}.
\end{multline*}
The substitution of all these approximations into \eqref{Q-approx} yields the desired relation:
\[
 Q_h^k(x,a)\approx r_h(x,a) + \sigma^{(N)}_h(\Pb_h(x,a), V_{h+1}^k).
\]
Using these ideas, we establish an upper bound on the regret:
\[
\text{Regret}(K)=\sum_{k=1}^K\big[V_1^{\pi_k}(x_1^k)-V_1^\star(x_1^k)\big],
\]
where  $\pi_k$ is the policy determined by Algorithm \ref{a:basic} in episode $k$.

\setcounter{theorem}{0}

\begin{theorem}
\label{t:main}
There exists an absolute constant $c_\beta{\ge2}$ such that, for any fixed $p\in (0,1)$, if we set {$\lambda=1$ and} $\beta=c_\beta\, d H\sqrt{\varGamma}$ in Algorithm \ref{a:basic}, with $\varGamma=\ln(2dKH/p)$, then with probability at least $1-p$
\[
\text{\rm Regret}(K)\le 3\beta\,S_H(N)\sqrt{2dK\ln(1+K)},\qquad\text{where}\quad S_H(N)=\sum_{h=0}^{H-1}N^h.
\]
\end{theorem}
The proof of Theorem \ref{t:main} is provided in the Appendix. 

{ Several comments are in order. The regret is $\widetilde{\mathcal{O}}\big(d^{3/2}H^2\sqrt{K}\big)$ for $N=1$, which is the rate of \cite{jin2020provably} for linear MDPs, and $\widetilde{\mathcal{O}}\big(d^{3/2}HN^{H-1}\sqrt{K}\big)$ for $N\ge 2$. The choice $\lambda=1$ is made for simplicity; for a fixed $\lambda\in(0,1)$ the same argument applies with constants depending on $\lambda$.

The factor $N$ in the recursion \eqref{delta-recursive} of Lemma \ref{lma57}, and hence the geometric factor $S_H(N)$, originates from a single property of the mini-batch mapping: for all $\Pb\in\Pc(\Xb)$ and all $W\ge0$,
\[
\sigma^{(N)}_h(\Pb,W)\le N\,\Pb W,
\]
which follows from $\varPsi_h(v)\le\max_{1\le j\le N}v_j\le\sum_{j=1}^Nv_j$ for $v\ge 0$. Beyond coherence, nothing else about $\sigma^{(N)}_h$ is used in the proof of Lemma \ref{lma57}. Consequently, if $c\ge1$ is a constant such that $\sigma^{(N)}_h(\Pb,W)\le c\,\Pb W$ for all $h\in[1:H-1]$, $\Pb\in\Pc(\Xb)$, and $W\ge0$, then  Theorem \ref{t:main} holds with $S_H(N)$ replaced by $S_H(c)$. By the dual representation \eqref{dual}, the smallest such $c$ is the supremum of $\|\xi\|_\infty$ over the dual densities $\xi$. For a mini-batch mapping \eqref{mini-batch}, $c\le c_{\rm b}$, where $c_{\rm b}$ is the corresponding constant of the base mapping $\sigma^{\rm b}$. With $\bar v=\frac1N\sum_{j=1}^N v_j$, we have the following examples.
\begin{tightitemize}
\item The mapping \eqref{batch-sup} is  a spectral risk measure \citep{acerbi2002spectral} with the spectrum $u\mapsto Nu^{N-1}$. The constant $N$ cannot be improved uniformly in $\Pb$.
\item For the mixture \eqref{Psi-mix}, $\varPsi(v)=(1-\varkappa)\bar v+\varkappa\max_jv_j\le\big(1+\varkappa(N-1)\big)\bar v$ for $v\ge0$, so that one may take $c=1+\varkappa(N-1)$.
\item For the mini-batch version of the Average Value at Risk \eqref{avar}, $c\le\min\{N,1/\alpha\}$.
\item For the mini-batch version of the mean--semideviation mapping \eqref{msd-form}, $c\le1+\kappa$, independently of $N$, because $(v_j-\bar v)_+\le v_j$ for $v\ge0$.
\end{tightitemize}
Exponential dependence of the regret on the horizon cannot be removed in general; for the entropic risk measure, \cite{fei2020risk} show it by a lower bound.}

\section{Application to the stochastic assignment problem}
\label{s:assignment}

We now illustrate our methods on an example for which a comparison with the theoretically optimal policy is possible.

\subsection{Problem formulation}

In the stochastic assignment problem, $H$ jobs with independent random values $C_1,\dots, C_H$ arrive sequentially at times $h=1,\dots,H$. The value $C_h$ has the distribution function $F_h(\Cdot)$ { with finite expectation}; we learn each job's value at the time of its arrival, but we do not know the values of future jobs.

At the beginning of the process, we are given $H$ random ``workers'' with weights $B_1^{(1)}\le B_1^{(2)}\le \dots \le B_1^{(H)}$.
The problem is to assign
one worker to each job, at the time of the job’s arrival, when the job's value becomes known, so that each worker is used exactly once and every job is processed. In the risk-neutral case, we want to maximize the following function:
\begin{equation}
 \Eb\bigg[\sum_{h=1}^H R(C_h,B_h^{(a_h)})\bigg],
\end{equation}
where the action $a_h$ is the rank of the worker assigned to the job $C_h$ at stage $h$. Here, we assume that the workers available at stage $h\in [1:H]$ are ranked according to their weights:
$B_h^{(1)}\le B_h^{(2)}\le \dots \le B_h^{(H-h+1)}$.
The action $a_h$ is understood as the rank of the worker assigned to the job $C_h$ in this ordered list; that is, the weight of the
assigned worker is $B_h^{(a_h)}$. The rewards $R(C_h,B_h^a)$ may be random, with distributions fully determined by $(C_h,B_h^a)$. We use the notation $r(C_h,B_h^a) = \Eb \big[ R(C_h,B_h^a)\,\big|\, C_h,B_h^a\big]$.

We define the state space at stage $h$ as $\Xb_h = \Rb^{H-h+2}$, $h\in [1:H]$; the state contains
the ordered weights of the workers available at stage $h$, and the value of the job $C_h$ that has just arrived:
\[
x_h = \big(B_h^{(1)},\dots,B_h^{(H-h+1)},C_h\big), \quad h\in [1:H].
\]
The action space at stage $h$ is $\Ab_h = [1:H-h+1]$. The next state is
\begin{equation}
\label{x-next}
x_{h+1} = \big(B_{h+1}^{(1)},\dots,B_{h+1}^{(H-h)},C_{h+1}\big),
\end{equation}
with
\[
B_{h+1}^{(j)}  = \begin{cases} B_h^{(j)} & \text{if } j\in [1:a_h-1], \\
B_h^{(j+1)} & \text{if } j\in [a_h: H-h],
\end{cases}
\]
and with $C_{h+1}$ being the value of the job at stage $h+1$.

\subsection{The learning model}

As in \cite{derman1972sequential}, we consider the reward function $r(C,B) = CB$.
In the expected value maximization problem, for the optimal policy, and more generally for every policy that depends on the
workers' weights only through their ranks, we have
\begin{equation}
    \label{Q-assign}
Q_h^\pi(x_h,a) = C_h B_h^{(a)} + \sum_{j=1}^{H-h} w_{h,j}^\pi B_{h+1}^{(j)},\quad h\in [1:H].
\end{equation}
The coefficient $w_{h,j}^\pi$ is the expected value of the job done by the $j$th remaining worker after the worker with rank $a$ has been used at stage $h$. This insight allows us to solve the expected-value problem analytically and derive the optimal policy, as described by \cite{derman1972sequential}.
The key observation is that by the Hardy--Littlewood--P\'olya inequality, the sum of the products
in \eqref{Q-assign} is largest when the sequences
\[
\big\{C_h,w_{h,1}^\pi, \dots, w_{h,H-h}^\pi\big\} \quad \text{and}\quad
\big\{ B_h^{(a)},B_{h+1}^{(1)},\dots,B_{h+1}^{(H-h)}\big\}
\]
have identical rank patterns. Thus, the best $a$ is the one for which $B_h^{(a)}$ has the same rank among the workers as  $C_h$ has among the $w_{h,j}$'s. This allows for the recursive calculation of the
optimal policy. The general formula is:
 \begin{equation}
 \label{optimal-weights}
 w_{h,j}^\pi = \underbrace{\int_{w_{h+1,j-1}^\pi}^{w_{h+1,j}^\pi} z\,dF_{h+1}(z)}_{\text{worker $j$ does job $C_{h+1}$}}
 + \underbrace{w_{h+1,j-1}^\pi\,\Pb\big[C_{h+1}\le w_{h+1,j-1}^\pi\big]}_{\text{worker lower than $j$ does job $C_{h+1}$}}
 + \underbrace{w_{h+1,j}^\pi\,\Pb\big[C_{h+1}> w_{h+1,j}^\pi\big]}_{\text{worker higher than $j$ does job $C_{h+1}$}}
;
 \end{equation}
 in the boundary cases,  the second or the third term disappears. 

Our first objective is to test whether Algorithm \ref{a:basic} with lazy learning can identify the optimal policy.

At stage $h$, we work with the
following $H-h+1$ features of the state--action pair:
\begin{equation}
    \label{features-assign}
\varphi_{hj}(x,a) =
\begin{cases} B_h^{(j)} & \text{if } j\in [1:a-1], \\
B_h^{(j+1)} & \text{if } j\in [a: H-h],\\
C_h B_h^{(a)} & \text{if } j=H-h+1.
\end{cases}
\end{equation}
It follows from \eqref{Q-assign} that the $Q$-factors can indeed be represented as linear functions of these features, as in \eqref{Q-w-exact}. We shall not use the theoretical insight about the correct values of the regression coefficients, but rather try to learn them by Algorithm \ref{a:basic}.

To make learning more challenging, we consider a version of the problem with Bernoulli rewards:
\[
R(C,B) = \begin{cases} C & \text{with probability $B$},\\ 0 & \text{otherwise.}
\end{cases}
\]
Its expected value is identical to the multiplicative reward function, and the theoretical feature weights should remain optimal.

In the test examples, we use $H=8$, job values uniformly in $(0,1)$ and initial worker weights sampled uniformly from $(0,1)$ as well. The method parameters were set to  $\lambda=0.1$ and $\beta= 0.1$.
 Because of the high variance of the rewards, the problem is rather difficult.

\begin{figure}[h!]
\centering{
\includegraphics[width=0.72\textwidth]{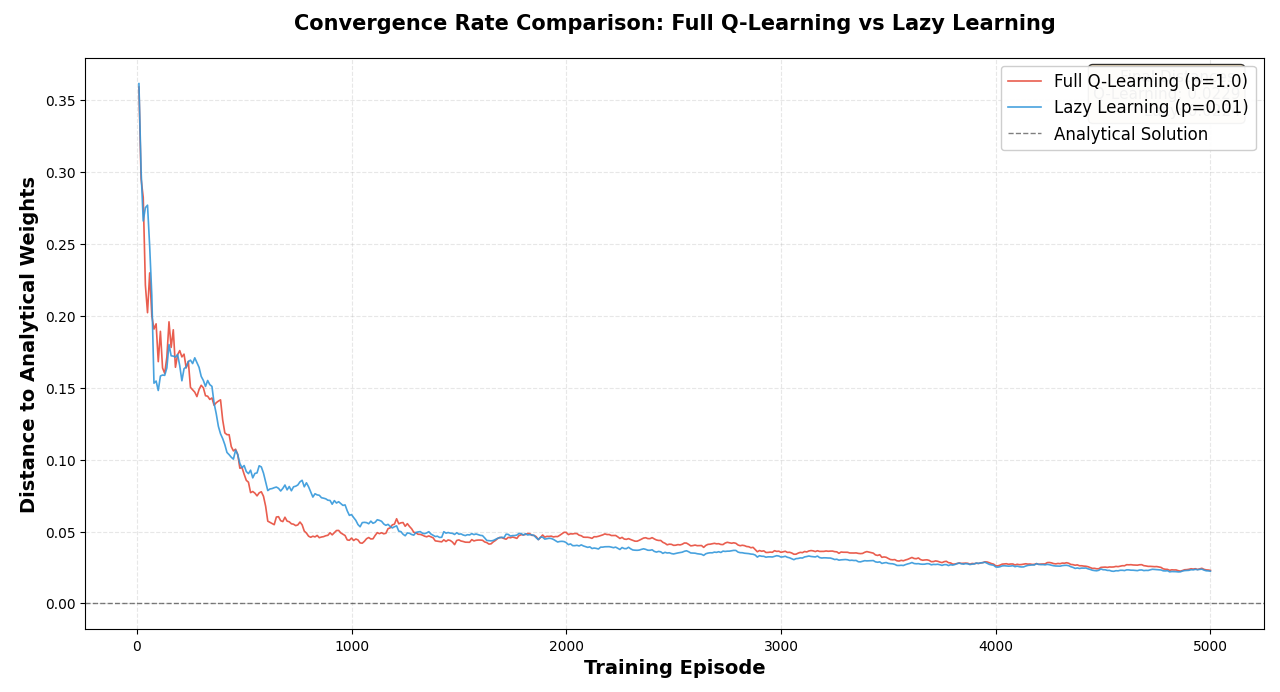}
  }
  \caption{The learning error $d(w^k,\hat{w})$ as a function of the number of episodes during full $Q$-learning and ``lazy'' learning with the batch size $N=1$ on an 8-stage stochastic assignment problem
  with Bernoulli rewards.}
   \label{f:configurations}
\end{figure}

Figure \ref{f:configurations} presents the distance of the feature weights computed by Algorithm \ref{a:basic} to the optimal weights given by \eqref{optimal-weights} in two versions
of Algorithm \ref{a:basic}: full $Q$-learning and ``lazy'' learning, with $p^{\textup{renew}} = 0.01$. Both versions of the method were trained on identical data over 5000 episodes. To remove
the dimension effect, the distance between
weight vectors was computed as follows:
\[
d(w,v) = \sqrt{ \frac{1}{H}\sum_{h=1}^H  \frac{1}{H-h+1}\sum_{j=1}^{H-h+1} (w_{h,j} - v_{h,j})^2}.
\]

Table \ref{tab:performance}
provides performance characteristics of the learned policies (evaluated over 10,000 validation episodes) and the learning times.
We can see from these results that with $N=1$ the ``lazy'' learning method is comparable to the full $Q$-learning algorithm, at about half  the training time. The performance of the two policies learned is statistically indistinguishable.
Still, their deviation from the optimal performance is  statistically significant.
The last two rows of Table \ref{tab:performance} compare the policies learned with the mini-batch size $N=2$. In this case, the deviation of full Q-learning from the optimum is significant at the 5\% level (t = 2.41), while that of lazy learning is not (t = 1.37). All computations were conducted on an AMD Ryzen AI Max+ processor.

\begin{table}[htbp]
\caption{Expected Performance of the Policies Learned for the Assignment Problem  }
\centering
\begin{tabular}{@{}lc@{}c@{}c@{}c@{}}
\toprule
\textbf{Method} & \textbf{Batch Size} &\textbf{Time (s)} & \textbf{Mean Reward} & \phantom{XXX}\textbf{{\boldmath $t$}-Statistic}\phantom{XXX}\\
\midrule
Analytical      & & \phantom{XXXX}   \phantom{XXXX}      &  2.4492 &  \\
Full Q-Learning & 1 & 1150 &  2.4475  &  5.1349  \\
Lazy Learning   & 1 & 603  &  2.4477  & 4.7339 \\
Full Q-Learning & 2 & 1907 &  2.4486  &  2.4063  \\
Lazy Learning   & 2 & 835  &  2.4488  & 1.3688 \\
\bottomrule
\end{tabular}
\label{tab:performance}
\end{table}

The convergence of the distance of the feature weights to the optimal values is illustrated in
Figure \ref{f:configurations_2}.

\begin{figure}[h!]
\centering{
\includegraphics[width=0.72\textwidth]{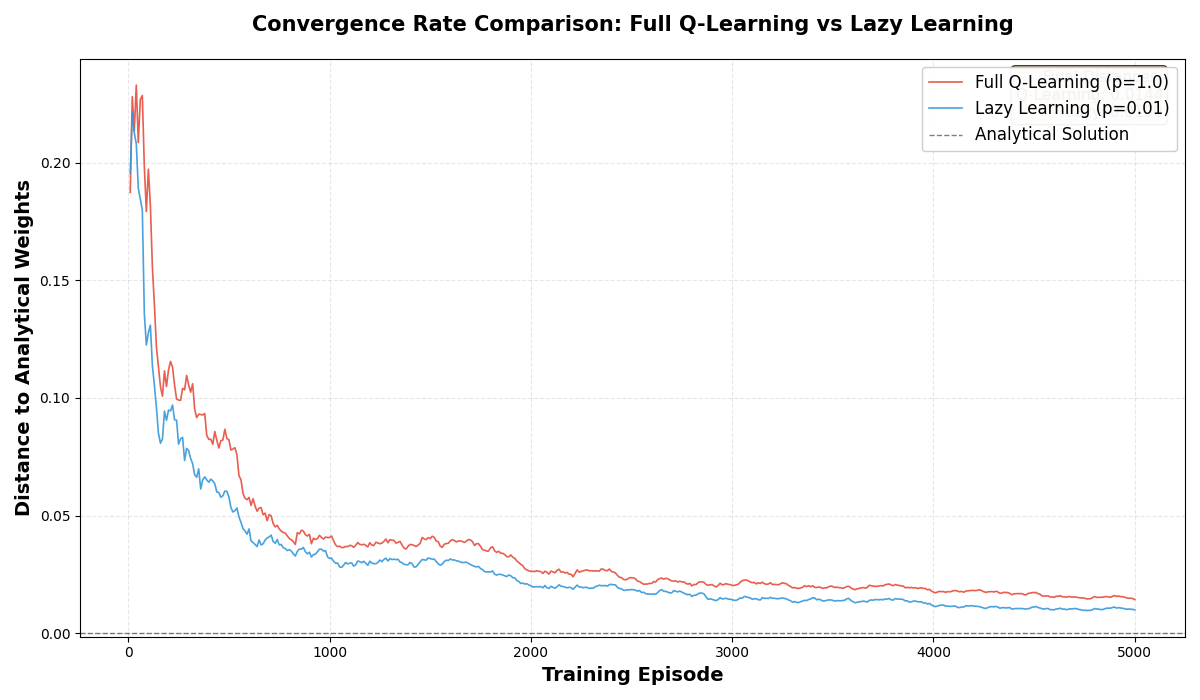}
  }
  \caption{The learning error $d(w^k,\hat{w})$ as a function of the number of episodes during full $Q$-learning and ``lazy'' learning with the batch size $N=2$ on an 8-stage stochastic assignment problem with Bernoulli rewards.}
   \label{f:configurations_2}
\end{figure}

In this experiment, lazy learning was not only cheaper but also more accurate. Presumably, avoiding the optimization step at every instance of the backward loop stabilizes the
method and facilitates a more accurate policy evaluation, which in turn results in a more effective optimization at the update steps.

\subsection{Risk-averse setting}

We use the mini-batch risk measure \eqref{batch-sup} with $N=2$,  mixed with the conditional expectation, as in \eqref{Psi-mix}.

In our maximization setting, with batch observations $v_h^j = R_h^{j} + V_{h+1}( Y_{h+1}^{j})$, $j=1,2$,
we use the function
\begin{equation}
\label{example=Psi}
\varPsi_h(v_h^1,v_h^2) =(1 - \varkappa) (v_h^1 +v_h^2)/2 + \varkappa \min(v_h^1,v_h^2)
\end{equation}
for calculating the regression targets.

{ 
}
We still use, for every policy $\pi$, the approximation
\begin{equation}
\label{Q-w-approx}
Q_h^\pi(x,a)   \approx  \varphi(x,a)^\top w_h^\pi,
\end{equation}
with some weight vector $w_h^\pi$, and we want to learn the weights corresponding to the optimal policy.

{
The justification for the use of the feature-based approximation \eqref{Q-w-approx} is  the
following proposition.

\begin{proposition}
\label{p:all}
Suppose the rewards are deterministic: $R(C,B) = CB$.
Then, for all $h\in[1:H]$, all states $x_h$, and all actions $a$,
\begin{equation}
\label{Qh}
Q_h^\star(x_h,a)=C_hB_h^{(a)}+\sum_{j=1}^{H-h}\hat w_{h,j}\,B_{h+1}^{(j)},
\end{equation}
where $\hat w_H$ is the empty vector and, for $h\in[1:H-1]$ and $j\in[1:H-h]$,
\begin{equation}
\label{DLR-G}
\hat w_{h,j}
=\hat w_{h+1,j-1}+\int_{\hat w_{h+1,j-1}}^{\hat w_{h+1,j}}\big(1-F_{h+1}(s)\big)\big(1-\varkappa F_{h+1}(s)\big)\,\D s,
\end{equation}
with $\hat w_{h+1,0}=0$ and $\hat w_{h+1,H-h}=\infty$. The coefficients $\hat w_{h,j}$
are nondecreasing in $j$, and the optimal action at stage $h$ assigns $C_h$ to the worker
whose rank among $B_h^{(1)},\dots,B_h^{(H-h+1)}$ equals the rank of $C_h$ among
$\hat w_{h,1},\dots,\hat w_{h,H-h}$. Therefore, the risk-averse optimal policy is the
risk-neutral optimal policy for job values with the distribution functions
$G_h=F_h+\varkappa F_h(1-F_h)$, $h\in[1:H]$.
\end{proposition}

The proof of this result is provided in Appendix \ref{s:assignment_analysis}.

It follows from Proposition \ref{p:all} that for deterministic rewards our feature-based model is exact. For random rewards
it is a reasonable approximation.
}

We consider two cases: deterministic and Bernoulli. In the risk-neutral case, we use $\varkappa = 0$, and in the risk-averse case,
 $\varkappa = 0.5$. In both cases we use lazy learning with $p^{\textup{renew}} = 0.01$,  $\lambda=0.1$, and $\beta= 0.1$.

\begin{figure}[h!]
\centering{
\includegraphics[width=0.72\textwidth]{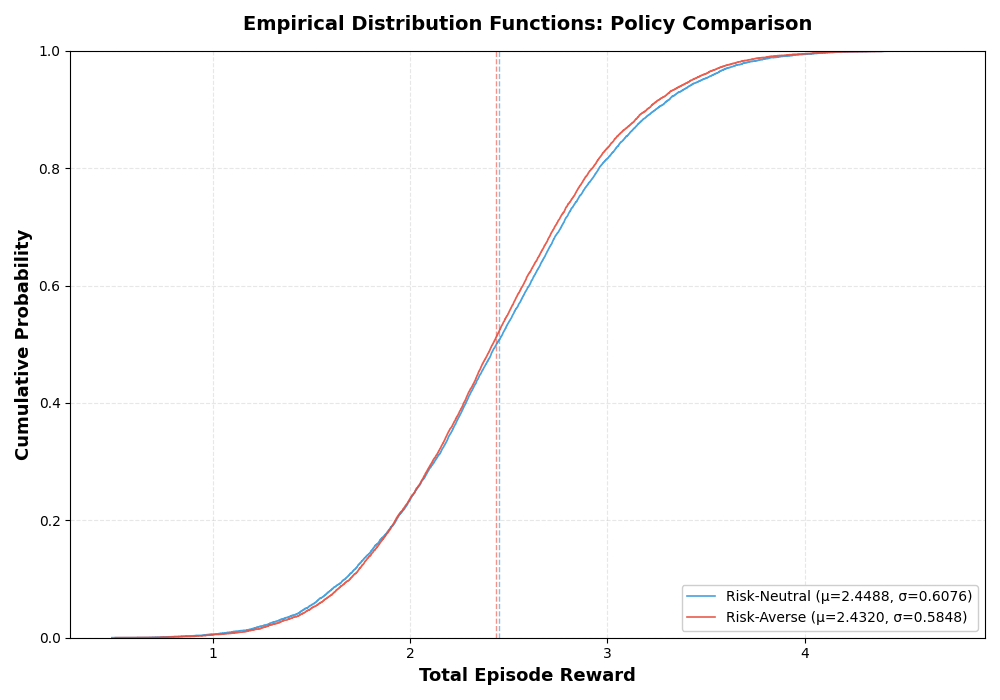}
  }
  \caption{The distribution functions of the performance of the policies on an 8-stage stochastic assignment problem in the risk-neutral and the risk-averse cases.}
\label{f:risk-averse}
\end{figure}

Figure \ref{f:risk-averse} presents the distribution functions of the learned expected-value policy and the risk-averse policy in the case of deterministic
rewards, evaluated over 10,000 independent test episodes. In the case of Bernoulli rewards, the graphs are very similar.

\begin{table}[h!]
\caption{Performance of the Learned Policies for the Assignment Problem}
\centering
\begin{tabular}{ccccc}
\toprule
\textbf{Method} & \textbf{Rewards} & \textbf{Time (s)} & \textbf{Average Reward} & \textbf{Standard Deviation}\\
\midrule
Risk-Neutral & Deterministic   & 718   & 2.4488  & 0.6076   \\
Risk-Averse & Deterministic     & 762  & 2.4320  & 0.5848  \\
Risk-Neutral & Bernoulli        & 717 &  2.4515  & 0.8822  \\
Risk-Averse & Bernoulli        & 762  & 2.4273  & 0.8713   \\
\bottomrule
\end{tabular}
\label{t:risk_performance}
\end{table}

 Table \ref{t:risk_performance} summarizes the results obtained. The learning time is higher than
in Table \ref{tab:performance}, because a batch of $N=2$ observations is generated at each forward step and used in regression. This also explains the high accuracy of the expected-value solutions. Again, in the mini-batch version, the lazy learning method is statistically indistinguishable from the analytical solution and still significantly cheaper.

We can see from these results that the risk-averse learning method is effective and robust. It finds policies that trade expected performance for a lower standard deviation: a feature expected of this kind of modeling approach.
The differences are small  because risk is difficult to control in this problem.

\section{Application to a short-horizon multi-armed bandit problem}
\label{s:bandit}

In this section, we consider a multi-armed bandit problem with a short horizon and substantial uncertainty.

At each stage $h = 1, \ldots, H$, a decision-maker must allocate a fixed capital $C$ across $n$ arms. Each arm~$j$ has an unknown reward distribution characterized by an unknown mean $\mu_j$ and an unknown variance $\sigma_j^2$. The decision-maker learns about these parameters through Bayesian updating based on observed rewards. The capital $C$ is an integer, and the allocations are restricted to integers.

Problems of this type have been studied extensively; see \citep{tran2012knapsack} and \cite[Sec. 7.6]{lattimore2020bandit}. Our example differs in that we consider a short horizon rather than asymptotic performance, which involves substantial risk.

\subsection{Problem formulation}

At each stage $h$, the decision-maker chooses an allocation vector
$a_h = (a_{h1}, a_{h2}, \ldots, a_{hn}) \in \Ab$,
where the feasible action set is
$\Ab = \left\{a \in \Zb_+^n : \sum_{j=1}^n a_j = C\right\}$.
When arm $j$ is pulled with allocation $a_{hj} > 0$,
a single observation $R_{hj} \sim \Nc(\mu_j, \sigma_j^2)$ is drawn, and
the immediate payoff from arm $j$ is $a_{hj}  R_{hj}$. Thus, the
total immediate reward at stage $h$ is:
\[
 R(a_h) = \sum_{j: a_{hj} > 0} a_{hj} R_{hj}.
 \]
In this model, the random rewards do not depend on the state, which is the information state to be described below.
The allocation amount $a_{hj}$  does not affect information gain---we observe $R_{hj}$ as long as $a_{hj}>0$.


Since the true arm parameters $(\mu_j, \sigma_j^2)$ are unknown, we maintain posterior distributions based on observed rewards. We use a normal-inverse-gamma conjugate prior; see
\cite[Ch. 3]{gelman2013bayesian}.
At stage $h$, the state is:
\[
x_h = (\mu_{h1}, \kappa_{h1}, \alpha_{h1}, \beta_{h1}, \ldots, \mu_{hn}, \kappa_{hn}, \alpha_{hn}, \beta_{hn}) \in \Rb^{4n}.
\]
For each arm $j$, the four hyperparameters characterize the posterior distribution:
$\mu_{hj}$ is the mean,
$\kappa_{hj} > 0$ is the precision weight,
$\alpha_{hj} > 0$ and $\beta_{hj} > 0$ are shape and scale parameters for the inverse-gamma distribution of the variance. In this setting, we have
$\sigma_{j}^2 \sim \text{Inv-Gamma}(\alpha_{hj}, \beta_{hj})$ and
$\mu_{j} \mid \sigma_{j}^2 \sim \Nc\big(\mu_{hj}, {\sigma_{j}^2}/{\kappa_{hj}}\big)$.  
{ The predictive distribution of the reward $R_j$ of arm $j$ is the Student distribution with $2\alpha_{hj}$ degrees of freedom, location $\mu_{hj}$, and  variance
\[
v_{hj}=\frac{\beta_{hj}(\kappa_{hj}+1)}{\kappa_{hj}(\alpha_{hj}-1)} .
\]
}

{
After taking action $a_h$ and observing rewards, the state is updated by the  normal--inverse-gamma posterior formulas.
If $a_{hj} > 0$ (arm $j$ is pulled), we
observe $R_{hj}$ and update:
\[
\begin{aligned}
\kappa_{h+1,j} &= \kappa_{hj} + 1, \\
\mu_{h+1,j} &=  \mu_{hj} + \frac{ R_{hj} - \mu_{hj}}{\kappa_{hj} + 1}, \\
\alpha_{h+1,j} &= \alpha_{hj} + \frac{1}{2}, \\
\beta_{h+1,j} &= \beta_{hj} + \frac{\kappa_{hj}(R_{hj} - \mu_{hj})^2}{2(\kappa_{hj} + 1)}.
\end{aligned}
\]
If $a_{hj} = 0$ (arm $j$ is not pulled), the parameters associated with arm $j$ remain unchanged. }

\subsection{The learning model}

We approximate the $Q$-functions linearly, as in \eqref{Q-w-approx}. Permutation invariance with respect to arm labels is
achieved by  reordering the arms {
by  $\mu_{hj}$ in decreasing order, with ties broken by $a_{hj}$ in decreasing order. We denote the resulting permutation by $\ell(\Cdot)$.}

Let $\hat{\sigma}^2_{hj}$ be the hypothetical posterior mean of $\sigma_j^2$ if $R_{hj} = \mu_{hj}$ were observed:
$
\hat{\sigma}_{hj}^2 = 
{\beta_{hj}}/({\alpha_{h+1,j} - 1)}
$.
{
The first $2n$ features are the arm-dependent parameters (re-ordered by $\ell(\Cdot)$):
\[
\varphi_{[1:2n]}(x_h, a_h)^\top = \begin{bmatrix}
\mu_{h,\ell(1)} ,
\cdots  ,
\mu_{h,\ell(n)} ,
\hat{\sigma}_{h,\ell(1)} ,
\cdots ,
\hat{\sigma}_{h,\ell(n)}
\end{bmatrix}.
\]
In addition to that, we introduce the immediate reward's mean and standard deviation:
\[
\varphi_{2n+1}(x_h, a_h) = \sum_{j=1}^n a_{hj} \mu_{hj},\qquad
\varphi_{2n+2}(x_h, a_h)=\Big(\sum_{j=1}^na_{hj}^2\,v_{hj}\Big)^{1/2}.
\]
Our last feature is  the \emph{knowledge gradient}:
\begin{equation}
\label{KG}
\varphi_{2n+3}(x_h, a_h) = \Eb\Big[\max_{j\in[1:n]}\mu'_{hj}\Big]-\max_{j\in[1:n]}\mu_{hj},
\end{equation}
with 
$\mu'_{hj} = \mu_{h+1,j}$ being the posterior mean after the observation, regarded under the predictive distribution given $(x_h, a_h)$.
For a single arm, \eqref{KG} is the knowledge gradient of \cite{gupta1996bayesian} and \cite{frazier2008knowledge}; here the arms in a selected subset are observed simultaneously. 

The knowledge gradient can be calculated at each state.
If $a_{hj}>0$, then $\mu'_{hj}=\mu_{hj}+(R_{hj}-\mu_{hj})/(\kappa_{hj}+1)$ has the Student
distribution with $2\alpha_{hj}$ degrees of freedom, location $\mu_{hj}$, and variance
$v_{hj}/(\kappa_{hj}+1)^2$; we denote its distribution function by $F_{hj}$. If $a_{hj}=0$,
then $\mu'_{hj}=\mu_{hj}$. With $S=\{j:a_{hj}>0\}$ and $m(a)=\max_{i\notin S}\mu_{hi}$
(note that $S\ne[1:n]$ if $C<n$), we have $\max_j\mu'_{hj}\ge m(a)$ and
\[
\Eb\Big[\max_{j\in[1:n]}\mu'_{hj}\Big]=m(a)+\int_{m(a)}^{\infty}\Big(1-\prod_{j\in S}F_{hj}(z)\Big)\,\D z .
\]
}

In our learning experiments, before every episode, we sample  the arm means $\mu_j \sim U(0.3,0.7)$ and the standard deviations $\sigma_j \sim U(0.15,0.25)$, independently of each other.
We set $C=3$, $n=5$, and $H=4$, so that the expected reward of a purely random policy is 6. We use lazy learning with $p^{\text{renew}} = 0.01$,
$\lambda=0.1$, $\beta=0.3$, $K=5000$ throughout.

In the risk-averse case, as in the previous section,
we use the mini-batch risk measure \eqref{batch-sup} with $N=2$. In our maximization setting, this corresponds to using the function
\eqref{example=Psi} for calculating the regression targets from each batch. 

Figure \ref{f:bandit_354} presents the distribution functions of the cumulative rewards resulting from two policies learned by Algorithm \ref{a:basic}: the risk-neutral one and the risk-averse one.
As in the previous example, the distribution functions were evaluated over 10,000 test episodes, independent of the training set.
In subfigure (a),
we set a highly non-informative initial state $\mu_{1j} = 1$, $\kappa_{1j} = 1$, $\alpha_{1j} = 2$, $\beta_{1j} = 1$ for all arms $j$,
corresponding to an initial variance estimate of $\hat{\sigma}_{1j}^2 = 1$. In subfigure (b), the initial information state is centered:
$\mu_{1j} = 0.5$, $\kappa_{1j} = 1$, $\alpha_{1j} = 2$, $\beta_{1j} = 0.04$ for all arms~$j$. This corresponds to an expected mean 0.5 and
an expected variance 0.04 of the arm distributions sampled before each episode.

\begin{figure}[h!]
\centering
   \begin{subfigure}[b]{0.9\textwidth}
   \centering
   \includegraphics[width=0.8\textwidth]{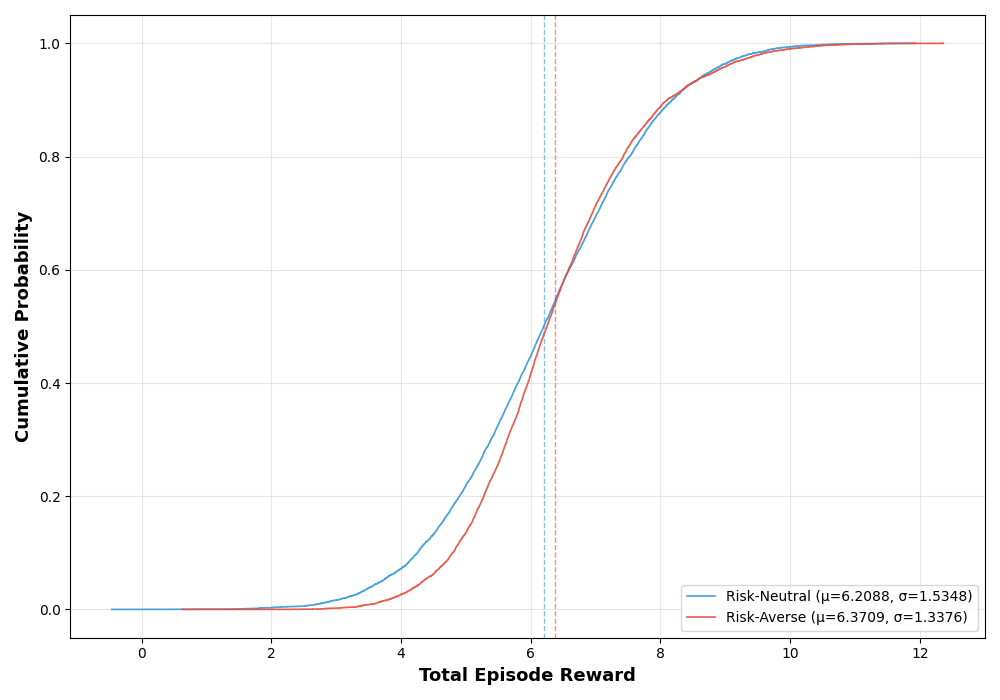}
   \caption{The initial state is non-informative: (1,1,2,1).}
   \label{fig:Ng1}
   \end{subfigure}
 \vspace{1ex}\\
   \begin{subfigure}[b]{0.9\textwidth}
   \centering
   \includegraphics[width=0.8\textwidth]{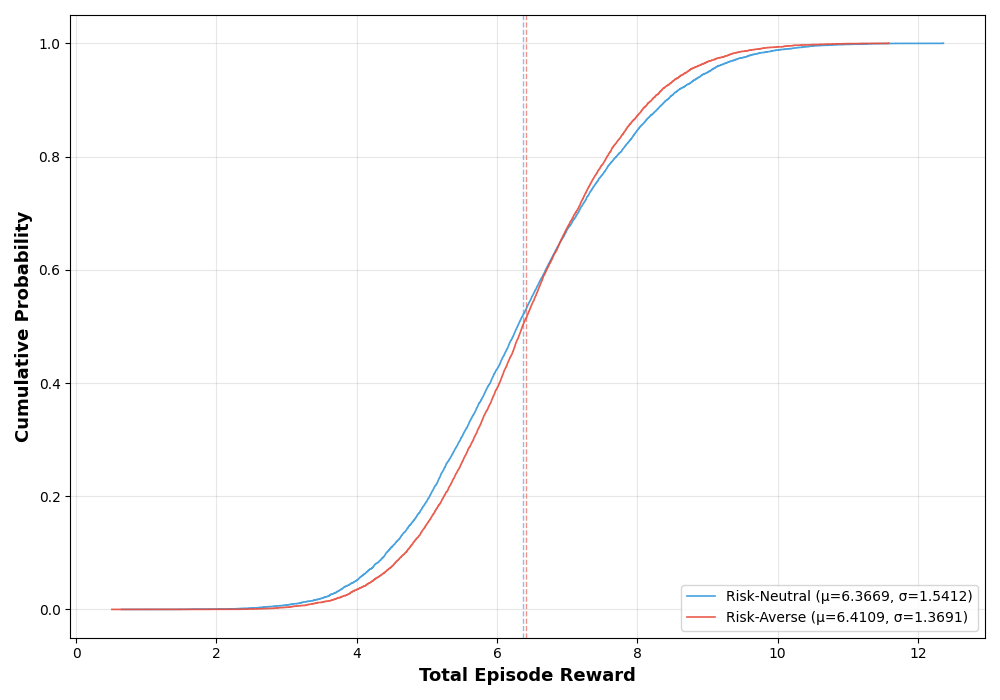}
   \caption{The initial  state is centered: (0.5,1,2,0.04).}
   \label{fig:Ng2}
   \end{subfigure}

\caption{The cumulative rewards of the learned policies in a multi-armed bandit problem with $C=3$, $n=5$, and $H=4$, in the risk-neutral and the risk-averse cases.
}
\label{f:bandit_354}
\end{figure}

An interesting conclusion from these experiments is that the risk-averse solution outperforms the risk-neutral solution not only in the standard deviation but sometimes also in the expected value.
This is probably because our feature-based model is crude.  Presumably, incorporating risk aversion into learning makes the solution more robust with respect to model imperfections. This phenomenon is more pronounced in case (a), where a highly
non-informative initial state was used. A similar observation was made by \cite{kose2021risk} in connection with risk-averse temporal difference learning. However, when the number of stages increases, this advantage
diminishes, and the usual trade-off between mean and standard deviation can be observed. This can be seen from 
Table \ref{t:bandit_performance}, which summarizes the performance of the learned policies in the 8-stage case. After the best arm is learned, the remaining risk is the reward noise, which can be reduced only by diversifying the allocation, at a cost in the mean.

\begin{table}[h]
\caption{Performance of the Learned Policies for the 8-Stage Bandit Problem}
\centering
\begin{tabular}{cccccc}
\toprule
 \textbf{Method}  & \textbf{Average Reward} & \textbf{Standard Deviation}\\
\midrule
 Risk-Neutral      & 13.0104  & 2.4105   \\
 Risk-Averse      & 12.9604  &  2.3438   \\
\bottomrule
\end{tabular}
\label{t:bandit_performance}
\end{table}

The main conclusion from our experiments reported in the last two sections is that our method is fast, reliable, and  applicable wherever classical $Q$-learning with linear features is.


\newpage
\appendix 
\setcounter{theorem}{0}

\section{Regret analysis}
\label{s:regret}

In what follows, we denote by $w_h^k$, $Q_h^k(\Cdot,\Cdot)$, and $V_h^k(\Cdot)$ the weight vector, the $Q$-factors, and the
value function approximation, respectively, calculated by the algorithm in episode $k$ at stage $h$. We use the notation
$\varPsi^{k,\tau}_{h+1} = \varPsi_h\big(\big\{ V_{h+1}^k( Y_{h+1}^{\tau,j})\big\}_{j\in [1:N]}\big)$.
Further, $w_h^{\pi_k}$, $Q_h^{\pi_k}(\Cdot,\Cdot)$, and $V_h^{\pi_k}(\Cdot)$ are the weights, $Q$-factors, and the value functions,
respectively, associated with the policy $\pi_k$ determined in episode $k$ at stage $h$.

We follow the general line of argument of \cite{jin2020provably}, with the modifications necessary for the risk-averse case.

We begin with several auxiliary results. The first one follows from the boundedness of the costs.

\begin{lemma}
\label{lma52}
For any policy $\pi$, let $\{w_h^\pi\}_{h\in[1:H]}$ be the corresponding weights such that $Q_h^\pi(x,a) = \langle \varphi(x,a), w_h^\pi \rangle$. Then $\|w_h^\pi\|\leq 2H\sqrt{d} $.
\end{lemma}
\begin{proof}
It follows from \eqref{multipattern1}--\eqref{multipattern2} that
 \[
Q_h^\pi(x,a) = r_h(x,a) +\sigma^{(N)}_h(\Pb_h(x,a), V_{h+1}^\pi)  = \langle \varphi(x,a), \theta_h + \rho_h(V_{h+1}^\pi) \rangle,
 \]
 and thus $\|w_h^\pi\| \le \|\theta_h\|+ \|\rho_h(V_{h+1}^\pi)\| $. By assumption, $\|\theta_h\|\le \sqrt{d}$. Furthermore, due to the boundedness of the costs, and the monotonicity and translation properties of
 the mini-batch measure of risk, we have
 \[
V_h^\pi(x) \le 1 + \sigma^{(N)}_h(\Pb_h(x,\pi(x)), V_{h+1}^\pi)
\le 1 + \sigma^{(N)}_h\Big(\Pb_h(x,\pi(x)), \1 \max_y V_{h+1}^{\pi}(y)\Big) \le 1 +  \max_y V_{h+1}^{\pi}(y).
 \]
Consequently, $V_h^\pi(\Cdot) \le H-h+1\le H$, and $\|\rho_h(V_{h+1}^\pi)\| \le H\sqrt{d}$ {(by the monotonicity and positive homogeneity of the coherent patterns, which give $0\le\rho_{hj}(V)\le H\rho_{hj}(\1)$ for all $0\le V\le H$)}.
\end{proof}

The second auxiliary result is a concentration inequality for mini-batch risk estimators.

In the algorithm, we shall use {$\lambda=1$ and} $\beta=c_\beta d H \sqrt{\varGamma}$ where {$\varGamma=\ln(2dKH/p)$}, and $c_\beta{\ge2}$ is a constant to be determined later. {We write $\varphi_h^\tau=\varphi(x_h^\tau,a_h^\tau)$.}
\begin{lemma}
\label{lma53}
There exists an absolute constant $C$ that is independent of $c_\beta$ such that for any fixed $p\in(0,1)$, if we let $E$ be the event that for all $k\in [1:K]$ and all $h\in [1:H]$:
\begin{equation*}
\bigg\|\sum_{\tau=1}^{k-1}
\varphi_h^\tau\Big[\varPsi^{k,\tau}_{h+1}-\sigma^{(N)}_h( \Pb_h(x_h^\tau,a_h^\tau), V_{h+1}^k)\Big]\bigg\|_{(\Lambda_h^k)^{-1}}\leq C\Cdot dH\sqrt{\chi},
\end{equation*}
where {$\chi=\ln[2(c_\beta+1)dKH/p]$}, then $\Pb(E)\geq 1-p/2$.
\end{lemma}
\begin{proof}
{
For $h=H$ the left-hand side vanishes by the conventions adopted after Algorithm \ref{a:basic}; we fix $h\in[1:H-1]$. We follow the proof of \cite[Lem.~B.3]{jin2020provably} and indicate the points at which the risk-averse structure enters.

Consider a fixed $V\in\Vb$ with $0\le V\le H$,
and 
let
\[
\varepsilon^\tau(V)=\varPsi_h\big(\big\{V(Y_{h+1}^{\tau,j})\big\}_{j\in[1:N]}\big)-\sigma^{(N)}_h\big(\Pb_h(x_h^\tau,a_h^\tau),V\big).
\]
By the properties of the $\sigma$-algebras \eqref{filtration}, the vector $\varphi_h^\tau$ is $\Fc_h^\tau$-measurable and $\varepsilon^\tau(V)$ is $\Fc_h^{\tau+1}$-measurable. By \eqref{mini-batch-estimate}, $\Eb\big[\varepsilon^\tau(V)\,\big|\,\Fc_h^\tau\big]=0$. The aggregator $\varPsi_h$, being nondecreasing, translation equivariant, and null at $\0_N$, maps $[0,H]^N$ into $[0,H]$, and $\sigma^{(N)}_h(P,\Cdot)$ maps $[0,H]^{|\Xb|}$ into $[0,H]$ for the same reasons. Hence $|\varepsilon^\tau(V)|\le H$, and $\varepsilon^\tau(V)$ is conditionally sub-Gaussian with parameter $H$. The self-normalized inequality of \cite[Thm.~1]{abbasi2011improved} yields that for any $\delta\in(0,1)$, with probability at least $1-\delta$,  for all $k\in[1:K]$,
\begin{equation*}
\bigg\|\sum_{\tau=1}^{k-1}\varphi_h^\tau\,\varepsilon^\tau(V)\bigg\|_{(\Lambda_h^k)^{-1}}^2\le 2H^2\ln\bigg(\frac{\big(\det\Lambda_h^k\big)^{1/2}}{\delta}\bigg)\le H^2\big[d\ln(1+K)+2\ln(1/\delta)\big],
\end{equation*}
where we used $\det \Lambda_h^k\le(1+K/d)^d$.

Now we  estimate the sum of errors of the regression targets.
The function $V_{h+1}^k$ depends on the data of the episodes $1,\dots,k-1$, including the batches $Y_{h+1}^\tau$, for $\tau<k$, and the previous estimate does not apply to it directly. The regression targets $r_h+\varPsi^{k,\tau}_{h+1}$ lie in $[0,1+H]\subset[0,2H]$; hence, for every $v\in\Rb^d$,
\[
|v^\top w_h^k|\le2H\sum_{\tau=1}^{k-1}\big|v^\top(\Lambda_h^k)^{-1}\varphi_h^\tau\big|\le2H\sqrt{(k-1)\,v^\top(\Lambda_h^k)^{-1}v}\le2H\sqrt{K}\,\|v\|,
\]
that is, $\|w_h^k\|\le L=2H\sqrt{K}$; cf.  \cite[Lem.~B.2]{jin2020provably}. Owing to the truncation in Algorithm \ref{a:basic}, $V_{h+1}^k$ belongs to the class
\[
\Vc=\Big\{\textup{Proj}_{[0,H]}\Big(\min_{a\in\Ab}\Big[w^\top\varphi(\Cdot,a)-\beta\sqrt{\varphi(\Cdot,a)^\top A\,\varphi(\Cdot,a)}\Big]\Big):\ \|w\|\le L,\ A\succeq0,\ \|A\|\le1\Big\}.
\]
It differs from the class in \cite[Lem.~D.6]{jin2020provably} only by the signs of the bonus and of the optimization, which does not affect the covering estimate. Thus, for every $\epsilon>0$, $\Vc$ has an $\epsilon$-net $\Nc_\epsilon$ in $\|\Cdot\|_\infty$ with
\[
\ln|\Nc_\epsilon|\le d\ln(1+4L/\epsilon)+d^2\ln\big(1+8d^{1/2}\beta^2/\epsilon^2\big).
\]
Both mappings $V\mapsto\varPsi_h\big(\{V(Y^j)\}_{j\in[1:N]}\big)$ and $V\mapsto\sigma^{(N)}_h(P,V)$ are nondecreasing and translation equivariant, and therefore $1$-Lipschitz with respect to $\|\Cdot\|_\infty$. 
Hence $|\varepsilon^\tau(V)-\varepsilon^\tau(V')|\le2\|V-V'\|_\infty$. If $V\in\Vc$ and $V'\in\Nc_\epsilon$ satisfy $\|V-V'\|_\infty\le\epsilon$, then, as $\|\varphi_h^\tau\|\le1$ and $\Lambda_h^k\succeq I$,
\[
\bigg\|\sum_{\tau=1}^{k-1}\varphi_h^\tau\big[\varepsilon^\tau(V)-\varepsilon^\tau(V')\big]\bigg\|_{(\Lambda_h^k)^{-1}}\le2\epsilon\sum_{\tau=1}^{k-1}\|\varphi_h^\tau\|_{(\Lambda_h^k)^{-1}}\le2\epsilon K.
\]

We apply the estimate for a fixed value function with $\delta=p/(2H|\Nc_\epsilon|)$ to every $V'\in\Nc_\epsilon$ and every $h\in[1:H-1]$, and we set $\epsilon=dH/K$. With probability at least $1-p/2$, for all $k\in[1:K]$ and $h\in[1:H]$,
\[
\bigg\|\sum_{\tau=1}^{k-1}\varphi_h^\tau\,\varepsilon^\tau(V_{h+1}^k)\bigg\|_{(\Lambda_h^k)^{-1}}\le H\sqrt{d\ln(1+K)+2\ln(2H/p)+2\ln|\Nc_\epsilon|}+2dH.
\]
With $L=2H\sqrt K$, $\beta=c_\beta dH\sqrt{\varGamma}$, and $\epsilon=dH/K$,
\[
\ln|\Nc_\epsilon|\le d\ln\big(1+8K^{3/2}/d\big)+d^2\ln\big(1+8d^{1/2}c_\beta^2\,\varGamma K^2\big)\le C_1d^2\chi,
\]
with an absolute constant $C_1$, because each of $\ln(c_\beta+1)$, $\ln d$, $\ln K$, $\ln\varGamma$, and $\ln(2H/p)$ is at most $\chi$, and $\chi\ge\ln 6$. Since also $2dH\le2dH\sqrt{\chi}$, the assertion follows with an absolute constant $C$.

The argument uses only three properties of the estimates $\varPsi^{k,\tau}_{h+1}$: the conditional unbiasedness \eqref{mini-batch-estimate} for every fixed value function, values in $[0,H]$, and $1$-Lipschitz dependence on the value function in $\|\Cdot\|_\infty$. It therefore remains valid for any estimator $\widetilde{\varPsi}_{h+1}^{\tau}(Q_{h+1})$ that, as a function of $Q_{h+1}$, is determined by $Q_{h+1}$ and the data in $\Fc_h^{\tau+1}$, and that has these three properties. The economical (``lazy'') estimator is  a heuristic approximation to such an estimator; its analysis is beyond the scope of Theorem \ref{t:main}, because the stored actions $\tilde{a}_{h+1}^{\tau,j}$ depend on data observed after episode $\tau$. }
\end{proof}

\begin{lemma}
\label{lma55}
An absolute constant $c_\beta{\ge2}$ exists such that for $\beta=c_\beta\Cdot d H \sqrt{\varGamma}$ where {$\varGamma=\ln(2dKH/p)$}, on the event $E$ defined in Lemma \ref{lma53}, we have { for all policies $\pi$ and} for all $(x,a,h,k)\in \Xb\times \Ab\times[1:H]\times[1:K]$ that
\begin{equation*}
\langle \varphi(x,a), w_h^k\rangle -Q_h^\pi(x,a)=\sigma^{(N)}_h(\Pb_h(x,a), V_{h+1}^k) - \sigma^{(N)}_h(\Pb_h(x,a), V_{h+1}^\pi)+\Delta_h^k(x,a),
\end{equation*}
for some $\Delta_h^k(x,a)$ that satisfies $|\Delta_h^k(x,a)|\leq \beta\sqrt{\varphi(x,a)^{\top}(\Lambda_h^k)^{-1}\varphi(x,a)}$.
\end{lemma}
\begin{proof}
By \eqref{multipattern1}--\eqref{multipattern2}, there exists $w_h^\pi=\theta_h+\rho_h(V_{h+1}^\pi)$ such that for any $(x,a,h)\in \Xb\times \Ab\times[1:H]$:
\begin{equation*}
    Q_h^\pi(x,a)=r_h(x,a) + \sigma^{(N)}_h(\Pb_h(x,a), V_{h+1}^\pi) = \langle\varphi(x,a),w_h^\pi\rangle.
\end{equation*}
Then, for any $(x,a)\in \Xb\times \Ab$, we have
\begin{equation*}
\langle\varphi(x,a),w_h^\pi\rangle=\langle\varphi(x,a),\theta_h\rangle+\sigma^{(N)}_h(\Pb_h(x,a), V_{h+1}^\pi).
\end{equation*}
This further gives (with $r_h^\tau = r_h(x_h^\tau,a_h^\tau)$ and  $\varphi_h^\tau = \varphi(x_h^\tau,a_h^\tau)$)
\begin{equation}
\label{four-terms}
\begin{aligned}
\lefteqn{w_h^k-w_h^\pi=(\Lambda_h^k)^{-1}\sum_{\tau=1}^{k-1}\varphi_h^\tau\big[r_h^\tau+\varPsi^{k,\tau}_{h+1}\big]-w_h^\pi} \\
&=(\Lambda_h^k)^{-1}\bigg\{\sum_{\tau=1}^{k-1}\varphi_h^\tau\big[r_h^\tau+\varPsi^{k,\tau}_{h+1}\big]-\Big(\lambda I+\sum_{\tau=1}^{k-1}\varphi_h^\tau(\varphi_h^\tau)^{\top}\Big)w_h^\pi\bigg\} \\
&=(\Lambda_h^k)^{-1}\bigg\{ -\lambda w_h^\pi+\sum_{\tau=1}^{k-1}\varphi_h^\tau\big[r_h^\tau+\varPsi^{k,\tau}_{h+1}-\langle\varphi_h^\tau,w_h^\pi\rangle\big]\bigg\} \\
&=(\Lambda_h^k)^{-1}\bigg\{ -\lambda w_h^\pi+\sum_{\tau=1}^{k-1}\varphi_h^\tau \big[r_h^\tau+\varPsi^{k,\tau}_{h+1}-\langle \varphi_h^\tau,\theta_h\rangle
-\sigma^{(N)}_h\big( \Pb_h(x_h^\tau,a_h^\tau), V_{h+1}^\pi\big)\big]\bigg\} \\
&=-\underbrace{\lambda (\Lambda_h^k)^{-1}w_h^\pi}_{q_1}+\underbrace{(\Lambda_h^k)^{-1}\sum_{\tau=1}^{k-1}\varphi_h^\tau\big[\varPsi^{k,\tau}_{h+1}
-\sigma^{(N)}_h\big( \Pb_h(x_h^\tau,a_h^\tau), V_{h+1}^k\big)\big]}_{q_2}\\
&{\quad} + \underbrace{(\Lambda_h^k)^{-1}\sum_{\tau=1}^{k-1}\varphi_h^\tau\big[
\sigma^{(N)}_h\big( \Pb_h(x_h^\tau,a_h^\tau), V_{h+1}^k\big) -\sigma^{(N)}_h\big( \Pb_h(x_h^\tau,a_h^\tau), V_{h+1}^\pi\big)\big]}_{q_3} .
\end{aligned}
\end{equation}
To estimate $\langle \varphi(x,a), w_h^k - w_h^\pi\rangle$, we individually bound the three terms on the right-hand side. For the first term,
\begin{equation}
|\langle\varphi(x,a),q_1\rangle|=|\lambda\langle\varphi(x,a),(\Lambda_h^k)^{-1}w_h^\pi\rangle|\leq \sqrt{\lambda}\|w_h^\pi\|\sqrt{\varphi(x,a)^{\top}(\Lambda_h^k)^{-1}\varphi(x,a)}.
\end{equation}
For the second term, on the event $E$ defined in Lemma \ref{lma53}, we have
\begin{equation}
|\langle\varphi(x,a),q_2\rangle|\leq c_0\Cdot dH\sqrt{\chi}\sqrt{\varphi(x,a)^{\top}(\Lambda_h^k)^{-1}\varphi(x,a)},
\end{equation}
for an absolute constant $c_0$ independent of $c_\beta$, and $\chi=\log[2(c_\beta+1)d{KH}/p]$. For the third term,
\begin{multline*}
\langle\varphi(x,a),q_3\rangle =\Big\langle\varphi(x,a),(\Lambda_h^k)^{-1}\sum_{\tau=1}^{k-1}\varphi_h^\tau\big[
\sigma^{(N)}_h\big( \Pb_h(x_h^\tau,a_h^\tau), V_{h+1}^k\big) -\sigma^{(N)}_h\big( \Pb_h(x_h^\tau,a_h^\tau), V_{h+1}^\pi\big)\big]\Big\rangle\\
=\Big\langle\varphi(x,a),(\Lambda_h^k)^{-1}\sum_{\tau=1}^{k-1}\varphi_h^\tau(\varphi_h^\tau)^{\top}\big(\rho_h(V_{h+1}^k)-\rho_h(V_{h+1}^\pi)\big)\Big\rangle \\
=\underbrace{\big\langle\varphi(x,a),\rho_h(V_{h+1}^k)-\rho_h(V_{h+1}^\pi)\big\rangle}_{p_1}-\underbrace{\lambda\big\langle\varphi(x,a),(\Lambda_h^k)^{-1}\big(\rho_h(V_{h+1}^k)-\rho_h(V_{h+1}^\pi)\big)\big\rangle}_{p_2}.
\end{multline*}
By \eqref{multipattern1} we have
\[
p_1 = \big\langle\varphi(x,a),\rho_h(V_{h+1}^k)-\rho_h(V_{h+1}^\pi)\big\rangle = \sigma^{(N)}_h(\Pb_h(x,a),V_{h+1}^k) - \sigma^{(N)}_h(\Pb_h(x,a),V_{h+1}^\pi).
\]
Furthermore, as in the proof of Lemma \ref{lma52}, $\|\rho_h(V_{h+1}^\pi)\|\le H\sqrt{d}$. 
{ The same bound holds for $\|\rho_h(V_{h+1}^k)\|$: the truncation in Algorithm \ref{a:basic} gives $0\le V_{h+1}^k\le H$ deterministically, whence $0\le\rho_{hj}(V_{h+1}^k)\le H$ by the monotonicity, translation, and positive homogeneity of the coherent patterns, and $\|\rho_h(V_{h+1}^k)\|\le H\|\rho_h(\1)\|\le H\sqrt{d}$.}
Thus, by the Cauchy--Schwarz inequality,
\[
 |p_2|\leq 2H\sqrt{d\lambda} \sqrt{\varphi(x,a)^{\top}(\Lambda_h^k)^{-1}\varphi(x,a)}.
\]
This yields the bound for the third term:
\[
\big| \langle\varphi(x,a),q_3\rangle - \sigma^{(N)}_h(\Pb_h(x,a),V_{h+1}^k) + \sigma^{(N)}_h(\Pb_h(x,a),V_{h+1}^\pi)\big| \le 2H\sqrt{d\lambda} \sqrt{\varphi(x,a)^{\top}(\Lambda_h^k)^{-1}\varphi(x,a)}.
\]
Finally, since $\langle\varphi(x,a),w_h^k-w_h^\pi\rangle=\langle\varphi(x,a),-q_1+q_2+q_3\rangle$, we have:
\begin{multline}
\big|\langle \varphi(x,a), w_h^k\rangle-Q_h^\pi(x,a)- \sigma^{(N)}_h(\Pb_h(x,a),V_{h+1}^k) + \sigma^{(N)}_h(\Pb_h(x,a),V_{h+1}^\pi)\big|\\
\leq c' d\sqrt{\chi}H\sqrt{\varphi(x,a)^{\top}(\Lambda_h^k)^{-1}\varphi(x,a)},
\end{multline}
for an absolute constant $c'$ independent of $c_\beta$. It remains to show that an absolute constant $c_\beta\geq 2$ exists such that
\begin{equation}
c'\sqrt{\varGamma+\ln(c_\beta+1)}\leq c_\beta\sqrt{\varGamma}
\end{equation}
for any $\varGamma\in[\ln 2,\infty)$.
This can be done by picking $c_\beta$ that satisfies the condition
\[
c'\sqrt{\ln 2+\ln(c_\beta+1)}\leq c_\beta\sqrt{\ln 2}.
\]
{Three remarks complete the proof. First, $\lambda\le1$ was used to absorb the
terms $\sqrt{\lambda}\,\|w_h^\pi\|\le2H\sqrt{d\lambda}$ and $|p_2|\le2H\sqrt{d\lambda}$
into $c'dH\sqrt{\chi}$. Second, the estimate is simultaneous in $\pi$: the term $q_2$
does not involve $\pi$, while the bounds on $q_1$ and $p_2$ use only
$\|w_h^\pi\|\le2H\sqrt{d}$ and $\|\rho_h(V_{h+1}^\pi)\|\le H\sqrt{d}$, which hold
uniformly over all policies by Lemma \ref{lma52}. On the event $E$ the conclusion is
therefore a deterministic statement about all policies at once, and it may be evaluated
at the data-dependent policy $\pi_k$. Third, at the terminal
stage $h=H$ the regression target contains no aggregator term, the sums defining $q_2$
and $q_3$ are absent, and the asserted identity holds with
$\Delta_H^k=-\langle\varphi(x,a),q_1\rangle$ and vanishing $\sigma^{(N)}$-terms, by the
conventions adopted after Algorithm \ref{a:basic}.}
\end{proof}

\begin{lemma}
\label{lma56}
{On the event $E$ defined in Lemma \ref{lma53},} we have $Q_h^k(x,a)\leq Q_h^\star(x,a)$ for all $(x,a,h,k)\in \Xb\times \Ab\times [1:H]\times[1:K]$.
\end{lemma}
\begin{proof}
We use backward induction. At the last stage $H$, according to Lemma \ref{lma55}, we have
\[
\langle\varphi(x,a),w_H^k\rangle-Q^\star_H(x,a)\leq\beta\sqrt{\varphi(x,a)^{\top}(\Lambda_H^k)^{-1}\varphi(x,a)}.
\]
By the nonnegativity of the costs, $Q^\star_H(x,a)\ge 0$, and thus
\[
Q^\star_H(x,a)\geq \max\big\{\langle\varphi(x,a),w_H^k\rangle- \beta\sqrt{\varphi(x,a)^{\top}(\Lambda_H^k)^{-1}\varphi(x,a)},0\big\} { \,\ge\, } Q_H^k(x,a).
\]
If the statement holds true at stage $h+1$, then at step $h$, also by Lemma \ref{lma55}, we have
\[
 Q_h^\star(x,a)\geq\langle \varphi(x,a), w_h^k\rangle + \sigma^{(N)}_h(\Pb_h(x,a),V_{h+1}^\star)  - \sigma^{(N)}_h(\Pb_h(x,a),V_{h+1}^k)   -\beta\sqrt{\varphi(x,a)^{\top}(\Lambda_h^k)^{-1}\varphi(x,a)}.
\]
By the induction hypothesis and the monotonicity of
$\sigma^{(N)}_h(\Pb_h(x,a), \,\Cdot)$, we have  $\sigma^{(N)}_h(\Pb_h(x,a),V_{h+1}^\star)  \ge  \sigma^{(N)}_h(\Pb_h(x,a),V_{h+1}^k)$. As a result,
\[
Q^\star_h(x,a)\geq \max\big\{\langle\varphi(x,a),w_h^k\rangle-\beta\sqrt{\varphi(x,a)^{\top}(\Lambda_h^k)^{-1}\varphi(x,a)},0\big\} { \,\ge\,} Q_h^k(x,a).
\]
By induction, the statement holds for all $h$.
\end{proof}

In the following lemma, the specific construction of the mini-batch coherent transition risk mappings plays a key role.
\begin{lemma}
\label{lma57}
Let $\delta_h^k=V_h^{\pi_k}(x_h^k)-V_h^k(x_h^k)$. Then, on the event $E$ defined in Lemma $\ref{lma53}$,
we have the following recursive estimate for any $(k,h) \in[1:K]\times[1:H]$:
\begin{equation}
\label{delta-recursive}
\delta_h^k\leq N \delta_{h+1}^k  + N \zeta_{h+1}^k + { 2} \beta \sqrt{(\varphi_h^{k})^{\top}(\Lambda_h^k)^{-1}\varphi_h^{k}},
\end{equation}
with $\zeta_{h+1}^k = \Eb \big[ \delta_{h+1}^k \big| {\Fc_h^k}\big] - \delta_{h+1}^k$. For $h=H$ the $(H+1)$-terms are null.
\end{lemma}
\begin{proof}
{If the upper truncation binds at $(x_h^k,a_h^k)$, that is,
$Q_h^k(x_h^k,a_h^k)=H$, then $\delta_h^k\le0$, because $Q_h^{\pi_k}\le H$ by Lemma
\ref{lma52}. The right-hand side of \eqref{delta-recursive} is nonnegative on the event
$E$ because $\delta_{h+1}^k+\zeta_{h+1}^k=\Eb\big[\delta_{h+1}^k\,\big|\,\Fc_h^k\big]
=\Pb_h(x_h^k,a_h^k)\big(V_{h+1}^{\pi_k}-V_{h+1}^k\big)\ge0$ by Lemma \ref{lma56}, and
\eqref{delta-recursive} holds trivially. At the terminal stage $h=H$, Lemma \ref{lma55}
with vanishing $\sigma^{(N)}$-terms gives
$\delta_H^k\le2\beta\sqrt{(\varphi_H^k)^{\top}(\Lambda_H^k)^{-1}\varphi_H^k}$, which is
\eqref{delta-recursive} with null $(H+1)$-terms. In the remaining cases,} by Lemma \ref{lma55}, for any $(k,h)\in[1:K]\times[1:H-1]$ we have:
\begin{align*}
\delta_h^k &= Q_h^{\pi_k}(x_h^k,a_h^k)-Q_h^k(x_h^k,a_h^k) \ {\le}\ 
Q_h^{\pi_k}(x_h^k,a_h^k) - \langle \varphi_h^k, w_h^k\rangle +  {  \beta\sqrt{(\varphi_h^k)^{\top}(\Lambda_h^k)^{-1}\varphi_h^k}}\\
&\le \sigma^{(N)}_h(\Pb_h(x_h^k,a_h^k), V_{h+1}^{\pi_k}) - \sigma^{(N)}_h(\Pb_h(x_h^k,a_h^k), V_{h+1}^k)
 +  { 2} \beta\sqrt{(\varphi_h^k)^{\top}(\Lambda_h^k)^{-1}\varphi_h^k}.
\end{align*}
To estimate the right-hand side, we use the  subadditivity
of the risk measure $\sigma^{(N)}_h(\Pb_h(x_h^k,a_h^k), \Cdot )$ and the definition of a mini-batch measure of risk. We have
\begin{multline}
\label{Psi-worst}
 \sigma^{(N)}_h(\Pb_h(x_h^k,a_h^k), V_{h+1}^{\pi_k}) - \sigma^{(N)}_h(\Pb_h(x_h^k,a_h^k), V_{h+1}^k) \le
\sigma^{(N)}_h(\Pb_h(x_h^k,a_h^k),V_{h+1}^{\pi_k}- V_{h+1}^k)\\
= \Eb \Big[ \varPsi\Big( \big\{(V_{h+1}^{\pi_k}- V_{h+1}^k)(Y_{h+1}^{k,j})\big\}_{j\in[1:N]}\Big)\,\Big| \,{\Fc_h^k}\Big].
\end{multline}
The risk aggregator $\varPsi(\Cdot)$ is nondecreasing, translation equivariant, and positively homogeneous. Therefore, with
$\1_N$ and ${\0}_N$ denoting the $N$-vectors of ones and zeros,
\begin{multline*}
\varPsi\Big( \big\{(V_{h+1}^{\pi_k}- V_{h+1}^k)(Y_{h+1}^{k,j})\big\}_{j\in[1:N]}\Big) \le
\varPsi\Big( \max_{1 \le j \le N}\big\{(V_{h+1}^{\pi_k}- V_{h+1}^k)(Y_{h+1}^{k,j})\big\}\1_N\Big)\\
=  \max_{1 \le j \le N}\big\{(V_{h+1}^{\pi_k}- V_{h+1}^k)(Y_{h+1}^{k,j})\big\}+ \varPsi(\0_N) 
=  \max_{1 \le j \le N}\big\{(V_{h+1}^{\pi_k}- V_{h+1}^k)(Y_{h+1}^{k,j})\big\}.
\end{multline*}
Recall that, by Lemma \ref{lma56}, $V_{h+1}^{\pi_k}(\Cdot) \ge V_{h+1}^{\star}(\Cdot) \ge V_{h+1}^k(\Cdot)$.
Now, with the definition of $x_{h+1}^k$ as a randomly selected component of the batch $Y_{h+1}^{k}$, we can combine the last three displayed
estimates as follows:
\begin{align*}
\delta_h^k &\le \Eb \Big[ \max_{1 \le j \le N}\big\{(V_{h+1}^{\pi_k}- V_{h+1}^k)(Y_{h+1}^{k,j})\big\} \,\Big|\, {\Fc_h^k}\Big] +
{ 2} \beta\sqrt{(\varphi_h^k)^{\top}(\Lambda_h^k)^{-1}\varphi_h^k}\\
&\le N \Eb \Big[ \frac{1}{N}\sum_{j=1}^N\big\{(V_{h+1}^{\pi_k}- V_{h+1}^k)(Y_{h+1}^{k,j})\big\} \,\Big|\, {\Fc_h^k}\Big] + { 2} \beta\sqrt{(\varphi_h^k)^{\top}(\Lambda_h^k)^{-1}\varphi_h^k}\\
&= N \Eb \Big[ (V_{h+1}^{\pi_k}- V_{h+1}^k)(x_{h+1}^{k}) \,\Big|\, {\Fc_h^k}\Big] + { 2} \beta\sqrt{(\varphi_h^k)^{\top}(\Lambda_h^k)^{-1}\varphi_h^k}.
\end{align*}
This is  \eqref{delta-recursive}.
The third line uses the uniform, value-independent selection of $x_{h+1}^k$
from the batch.
\end{proof}

Now we can provide the proof of the main theorem.

\begin{proof}[Proof of Theorem \ref{t:main}]
{
Throughout, $\lambda=1$, and we write $b_h^k=\sqrt{(\varphi_h^{k})^{\top}(\Lambda_h^{k})^{-1}\varphi_h^{k}}$.

On the event $E$ defined in Lemma \ref{lma53}, Lemma \ref{lma56} yields $V_1^k(x_1^k)\le V_1^\star(x_1^k)$, and therefore
$V_1^{\pi_k}(x_1^k)-V_1^\star(x_1^k)\le\delta_1^k$ for all $k\in[1:K]$. Lemma \ref{lma57} implies that, on $E$, for every $\nu\in[1:H]$,
\begin{equation}
\notag
\delta_\nu^k\le\sum_{h=\nu}^{H-1}N^{h-\nu+1}\zeta_{h+1}^k+2\beta\sum_{h=\nu}^{H}N^{h-\nu}\,b_h^k .
\end{equation}
This can be verified by backward induction.
Setting $\nu=1$ and summing over $k$, we obtain on $E$:
\begin{equation}
\label{regret-ineq}
\text{Regret}(K)\le\sum_{k=1}^K\sum_{h=1}^{H-1}N^{h}\zeta_{h+1}^k+2\beta\sum_{h=1}^{H}N^{h-1}\sum_{k=1}^K b_h^k .
\end{equation}

We estimate the two terms in \eqref{regret-ineq}. In the first term,
we order the pairs $(k,h)$ lexicographically and consider $D_{h+1}^k=N^h\zeta_{h+1}^k$. Each $D_{h+1}^k$ is $\Fc_{h+1}^k$-measurable and $\Eb\big[D_{h+1}^k\,\big|\,\Fc_h^k\big]=0$; thus $\{D_{h+1}^k\}$ is a martingale difference sequence with respect to the $\sigma$-algebras \eqref{filtration}. Moreover, $\delta_{h+1}^k=W_{h+1}^k(x_{h+1}^k)$, where the function $W_{h+1}^k=V_{h+1}^{\pi_k}-V_{h+1}^k$ is $\Fc_h^k$-measurable and takes values in $[-H,H]$, because $0\le V_{h+1}^{\pi_k}\le H$ (see the proof of Lemma \ref{lma52}) and $0\le V_{h+1}^k\le H$ (by the truncation in Algorithm \ref{a:basic}). 
Let $m_h^k=\Eb\big[\delta_{h+1}^k\,\big|\,\Fc_h^k\big]$. Both $\delta_{h+1}^k$ and $m_h^k$
lie in $[-H,H]$, and
\[
N^h\big(m_h^k-H\big)\le D_{h+1}^k=N^h\big(m_h^k-\delta_{h+1}^k\big)\le N^h\big(m_h^k+H\big).
\]
The endpoints of this interval, whose length is $2HN^h$, are $\Fc_h^k$-measurable.
Hoeffding's lemma \citep{hoeffding1963probability}, applied conditionally on
$\Fc_h^k$, gives
$\Eb\big[\exp(sD_{h+1}^k)\,\big|\,\Fc_h^k\big]\le\exp\big(s^2H^2N^{2h}/2\big)$
for all $s\in\Rb$.
 By the tower property and Chernoff's bound,
\[
\Pb\bigg\{\sum_{k=1}^K\sum_{h=1}^{H-1}D_{h+1}^k>t\bigg\}\le\exp\bigg(\frac{-t^2}{2KH^2\sum_{h=1}^{H-1}N^{2h}}\bigg),\quad t>0.
\]
As $\big(\sum_{h=1}^{H-1}N^{2h}\big)^{1/2}\le\sum_{h=1}^{H-1}N^{h}\le S_H(N)$, with probability at least $1-p/2$,
\begin{equation}
\label{martingale-term}
\sum_{k=1}^K\sum_{h=1}^{H-1}N^{h}\zeta_{h+1}^k\le H\sqrt{2K\ln(2/p)\textstyle\sum_{h=1}^{H-1}N^{2h}}\le H\,S_H(N)\sqrt{2K\ln(2/p)}.
\end{equation}
This estimate does not use the event $E$.

We now pass to the second (bonus) term of \eqref{regret-ineq}.
Since $\varphi(x,a)\in\Sb_d$, we have $\|\varphi(x,a)\|\le1$; as $\Lambda_h^k\succeq I$, also $b_h^k\le1$. From \cite[Lem.~11]{abbasi2011improved} we get
\[
\sum_{k=1}^K(b_h^k)^2\le 2\ln\big(\det(\Lambda_h^{K+1})\big),\qquad \Lambda_h^{K+1}=I+\sum_{\tau=1}^{K}\varphi_h^\tau(\varphi_h^\tau)^\top.
\]
By the arithmetic and geometric means inequality, 
$d \sqrt{\det(\Lambda_h^{K+1})} \le  \textup{tr}\big(\Lambda_h^{K+1}\big) \le d+K$, and thus
\[
\sum_{k=1}^K(b_h^k)^2 \le 2d\ln\Big(1+\frac{K}{d}\Big)\le2d\ln(1+K).
\]
Using the Cauchy--Schwarz inequality, we obtain
\[
 \sum_{k=1}^Kb_h^k\le\sqrt{2dK\ln(1+K)},
\]
and then
\begin{equation}
\label{bonus-term}
2\beta\sum_{h=1}^{H}N^{h-1}\sum_{k=1}^K b_h^k\le2\beta\,S_H(N)\sqrt{2dK\ln(1+K)}.
\end{equation}

Summing up, 
the event $E$ and the event in \eqref{martingale-term} occur jointly with probability at least $1-p$. On their intersection, \eqref{regret-ineq}--\eqref{bonus-term} give
\[
\text{Regret}(K)\le H\,S_H(N)\sqrt{2K\ln(2/p)}+2\beta\,S_H(N)\sqrt{2dK\ln(1+K)}.
\]
The second term is dominating here; for completeness, we provide a simple bound.
Since $\beta=c_\beta dH\sqrt{\varGamma}$ with $\varGamma=\ln(2dKH/p)\ge\ln(2/p)$, $d\ge1$, $\ln(1+K)\ge\ln2$, and $c_\beta\ge2$,
\[
H\sqrt{2K\ln(2/p)}=\beta\sqrt{2dK\ln(1+K)}\cdot\frac{\sqrt{\ln(2/p)}}{c_\beta d^{3/2}\sqrt{\varGamma\ln(1+K)}}\le\frac{\beta\sqrt{2dK\ln(1+K)}}{c_\beta\sqrt{\ln2}}\le\beta\sqrt{2dK\ln(1+K)}.
\]
Therefore, the first term in the last estimate of the regret does not exceed $\beta\,S_H(N)\sqrt{2dK\ln(1+K)}$, and $\text{Regret}(K)\le3\beta\,S_H(N)\sqrt{2dK\ln(1+K)}$, as asserted.} 
\end{proof}


{ 
\section{Analysis of the risk-averse stochastic assignment problem}
\label{s:assignment_analysis}

Suppose we use the aggregator \eqref{example=Psi},
where $v^i=C_h B_h^{(a)}+V_{h+1}(Y^i)$, $i=1,2$, and $Y^i$ are two independent observations of the next state. For simplicity we
assume that all jobs have the same distribution function $F$, but all our considerations cover the general case.

For $n\ge 1$ and a nondecreasing vector $t=(t_1,\dots,t_{n-1})$ with entries in
$[0,\infty)$, we set $t_0=0$, $t_n=\infty$, and
\begin{equation}
\label{m-def}
m_j^t(c)=\min\big\{\max\{c,t_{j-1}\},\,t_j\big\},\qquad j\in[1:n],\quad c\ge0,
\end{equation}
so that $m_1^t(c)=\min\{c,t_1\}$ and $m_n^t(c)=\max\{c,t_{n-1}\}$. Each $m_j^t(\Cdot)$ is
nondecreasing and $1$-Lipschitz, and $m_j^t\le m_{j+1}^t$. For any distribution function
$F$ of $C$ on $[0,\infty)$ with finite expected value,
\begin{equation}
\label{Em}
\Eb\big[m_j^t(C)\big]=t_{j-1}+\int_{t_{j-1}}^{t_j}\big(1-F(s)\big)\,\D s .
\end{equation}
Integrating by parts, we see that \eqref{Em} is the right-hand side of the recursion \eqref{optimal-weights}. 

The first lemma restates the Hardy--Littlewood--P\'olya  inequality in our notation.
\begin{lemma}
\label{l:value}
Let $0\le b_1\le\dots\le b_n$, let $t$ be as in \eqref{m-def}, and let $b^{(-a)}$ denote the vector $b$ with its $a$th entry removed. Then, for every $c\in[0,\infty)$,
\begin{equation}
\label{value}
\max_{a\in[1:n]}\Big\{c\,b_a+\sum_{j=1}^{n-1}t_j\,b^{(-a)}_j\Big\}=\sum_{j=1}^n b_j\,m_j^t(c),
\end{equation}
and the maximum is attained at $a=k$ whenever $c\in[t_{k-1},t_k]$. 
\end{lemma}

For an integrable random variable $X$ with independent copies $X^1,X^2$, define
\[
\varrho_\varkappa(X)=\Eb\Big[(1-\varkappa)\,\tfrac12\big(X^1+X^2\big)+\varkappa\min\big(X^1,X^2\big)\Big].
\]
This is the mini-batch mapping \eqref{mini-batch} built on \eqref{example=Psi}. 

\begin{lemma}
\label{l:comonotone}
Let $C$ have the distribution function $F$, let $g_1,\dots,g_n$ be nondecreasing
functions on $[0,\infty)$ such that $g_j(C)$ are integrable, and let $b\in\Rb^n_+$. Then
\begin{equation}
\label{comonotone}
\varrho_\varkappa\Big(\sum_{j=1}^n b_j g_j(C)\Big)=\sum_{j=1}^n b_j\,\Eb_G\big[g_j(C)\big],\qquad G=F+\varkappa F(1-F).
\end{equation}
\end{lemma}
\begin{proof}
The function $g=\sum_jb_jg_j$ is nondecreasing, so $\min\big(g(C^1),g(C^2)\big)=g\big(\min(C^1,C^2)\big)$. The distribution function of the minimum is $1-(1-F)^2$, and then $(1-\varkappa)F+\varkappa\big(1-(1-F)^2\big)=F+\varkappa F(1-F)$. 
\end{proof}
The distribution $G$ has a finite expected
value whenever $F$ does.

We can now summarize our analysis.

\vspace{1ex}
\begin{proof}[Proof of Proposition \ref{p:all}]
We prove by backward induction that \eqref{Qh} holds and that
\begin{equation}
\label{Vh}
V_h^\star(x_h)=\sum_{j=1}^{H-h+1}B_h^{(j)}\,m_j^{\hat w_h}(C_h)
\end{equation}
with a nondecreasing vector $\hat w_h$. At $h=H$ there is one worker, and
$V_H^\star(x_H)=C_HB_H^{(1)}=B_H^{(1)}m_1(C_H)$.

Let $h\le H-1$, and suppose \eqref{Vh} holds at stage $h+1$ with a nondecreasing
vector $\hat w_{h+1}$. Fix $x_h$ and $a$, and let $b=B_{h+1}(a)$. By \eqref{x-next}, the
weights $b$ are determined by $x_h$ and $a$, and the two observations of the next state
in \eqref{example=Psi} are $Y^i=(b,C^i)$, where $C^1,C^2$ are independent with the distribution
function $F$. The reward $C_hB_h^{(a)}$ is the same for both observations. By the
induction hypothesis, $V_{h+1}^\star(Y^i)=\sum_jb_j\,m_j^{\hat w_{h+1}}(C^i)$. Since
$\varPsi_h(r+v^1,r+v^2)=r+\varPsi_h(v^1,v^2)$, we obtain
\[
Q_h^\star(x_h,a)=C_hB_h^{(a)}+\varrho_{\varkappa}\Big(\sum_{j=1}^{H-h}b_j\,m_j^{\hat w_{h+1}}(C)\Big).
\]
The functions $m_j^{\hat w_{h+1}}$ are nondecreasing and $b\ge0$, so Lemma~\ref{l:comonotone}
 yields \eqref{Qh} with 
$\hat w_{h,j}=\Eb_{G}\big[m_j^{\hat w_{h+1}}(C)\big]$.
Then \eqref{DLR-G} follows from \eqref{Em} with $F$ replaced with $G$. The vector $\hat w_h$ is nondecreasing because
$m_j^{\hat w_{h+1}}\le m_{j+1}^{\hat w_{h+1}}$. Finally, $V_h^\star(x_h)=\max_aQ_h^\star(x_h,a)$, and,
by \eqref{x-next}, $B_{h+1}(a)=B_h^{(-a)}$.
Lemma~\ref{l:value} with $b=B_h$, $c=C_h$, and $t=\hat w_h$ gives \eqref{Vh} and the
rule for the optimal action.
\end{proof}

Consequently, the features \eqref{features-assign} represent the $Q$-factors exactly at
every stage: the coefficient of $C_hB_h^{(a)}$ is~$1$, and the coefficients of
$B_{h+1}^{(j)}$ are $\hat w_{h,j}$.
The risk-averse optimal policy is the
risk-neutral optimal policy for job values with the distribution function
$G=F+\varkappa F(1-F)$.
}





\bibliographystyle{abbrv}

\begin{thebibliography}{10}

\bibitem{abbasi2011improved}
Y.~Abbasi-Yadkori, D.~P{\'a}l, and C.~Szepesv{\'a}ri.
\newblock Improved algorithms for linear stochastic bandits.
\newblock {\em Advances in Neural Information Processing Systems}, 24, 2011.

\bibitem{acerbi2002spectral}
C.~Acerbi.
\newblock Spectral measures of risk: A coherent representation of subjective
  risk aversion.
\newblock {\em Journal of Banking \& Finance}, 26(7):1505--1518, 2002.

\bibitem{ArtznerDelbaenEberEtAl1999}
P.~Artzner, F.~Delbaen, J.-M. Eber, and D.~Heath.
\newblock Coherent measures of risk.
\newblock {\em Mathematical Finance}, 9(3):203--228, 1999.

\bibitem{ADEHK:2007}
P.~Artzner, F.~Delbaen, J.-M. Eber, D.~Heath, and H.~Ku.
\newblock Coherent multiperiod risk adjusted values and {B}ellman's principle.
\newblock {\em Annals of Operations Research}, 152:5--22, 2007.

\bibitem{auer2008near}
P.~Auer, T.~Jaksch, and R.~Ortner.
\newblock Near-optimal regret bounds for reinforcement learning.
\newblock {\em Advances in Neural Information Processing Systems}, 21, 2008.

\bibitem{ayoub2020model}
A.~Ayoub, Z.~Jia, C.~Szepesvari, M.~Wang, and L.~Yang.
\newblock Model-based reinforcement learning with value-targeted regression.
\newblock In {\em International Conference on Machine Learning}, pages
  463--474. PMLR, 2020.

\bibitem{azar2017minimax}
M.~G. Azar, I.~Osband, and R.~Munos.
\newblock Minimax regret bounds for reinforcement learning.
\newblock In {\em International Conference on Machine Learning}, pages
  263--272. PMLR, 2017.

\bibitem{basu2008learning}
A.~Basu, T.~Bhattacharyya, and V.~S. Borkar.
\newblock A learning algorithm for risk-sensitive cost.
\newblock {\em Mathematics of Operations Research}, 33(4):880--898, 2008.

\bibitem{bauerle2022markov}
N.~B{\"a}uerle and A.~Glauner.
\newblock Markov decision processes with recursive risk measures.
\newblock {\em European Journal of Operational Research}, 296(3):953--966,
  2022.

\bibitem{Borkar2001}
V.~Borkar.
\newblock A sensitivity formula for risk-sensitive cost and the actor–critic
  algorithm.
\newblock {\em Systems \& Control Letters}, 44(5):339--346, 2001.

\bibitem{Borkar2002}
V.~S. Borkar.
\newblock Q-learning for risk-sensitive control.
\newblock {\em Mathematics of Operations Research}, 27(2):294--311, 2002.

\bibitem{bradtke1996linear}
S.~J. Bradtke and A.~G. Barto.
\newblock Linear least-squares algorithms for temporal difference learning.
\newblock {\em Machine Learning}, 22(1-3):33--57, 1996.

\bibitem{bubeck2012regret}
S.~Bubeck, N.~Cesa-Bianchi, et~al.
\newblock Regret analysis of stochastic and nonstochastic multi-armed bandit
  problems.
\newblock {\em Foundations and Trends{\textregistered} in Machine Learning},
  5(1):1--122, 2012.

\bibitem{CDK:2006}
P.~Cheridito, F.~Delbaen, and M.~Kupper.
\newblock Dynamic monetary risk measures for bounded discrete-time processes.
\newblock {\em Electronic Journal of Probability}, 11:57--106, 2006.

\bibitem{cheridito2011composition}
P.~Cheridito and M.~Kupper.
\newblock Composition of time-consistent dynamic monetary risk measures in
  discrete time.
\newblock {\em International Journal of Theoretical and Applied Finance},
  14(01):137--162, 2011.

\bibitem{cherny2009new}
A.~Cherny and D.~Madan.
\newblock New measures for performance evaluation.
\newblock {\em The Review of Financial Studies}, 22(7):2571--2606, 2009.

\bibitem{Chow2014}
Y.~Chow and M.~Ghavamzadeh.
\newblock Algorithms for {CVaR} optimization in {MDP}s.
\newblock In {\em Advances in Neural Information Processing Systems}, pages
  3509--3517, 2014.

\bibitem{Chow2015a}
Y.~Chow, M.~Ghavamzadeh, L.~Janson, and M.~Pavone.
\newblock Risk-constrained reinforcement learning with percentile risk
  criteria.
\newblock {\em The Journal of Machine Learning Research}, 18(1):6070--6120,
  2017.

\bibitem{dann2017unifying}
C.~Dann, T.~Lattimore, and E.~Brunskill.
\newblock Unifying pac and regret: Uniform {PAC} bounds for episodic
  reinforcement learning.
\newblock {\em Advances in Neural Information Processing Systems}, 30, 2017.

\bibitem{dentcheva2023mini}
D.~Dentcheva and A.~Ruszczy{\'n}ski.
\newblock Mini-batch risk forms.
\newblock {\em SIAM Journal on Optimization}, 33(2):615--637, 2023.

\bibitem{dentcheva2024risk}
D.~Dentcheva and A.~Ruszczy{\'n}ski.
\newblock {\em Risk-Averse Optimization and Control: Theory and Methods}.
\newblock Springer International Publishing, Cham, Switzerland, 2024.

\bibitem{derman1972sequential}
C.~Derman, G.~J. Lieberman, and S.~M. Ross.
\newblock A sequential stochastic assignment problem.
\newblock {\em Management Science}, 18(7):349--355, 1972.

\bibitem{dong2019provably}
S.~Dong, B.~Van~Roy, and Z.~Zhou.
\newblock Provably efficient reinforcement learning with aggregated states.
\newblock {\em arXiv preprint arXiv:1912.06366}, 2019.

\bibitem{du2021bilinear}
S.~Du, S.~Kakade, J.~Lee, S.~Lovett, G.~Mahajan, W.~Sun, and R.~Wang.
\newblock Bilinear classes: A structural framework for provable generalization
  in {RL}.
\newblock In {\em International Conference on Machine Learning}, pages
  2826--2836. PMLR, 2021.

\bibitem{fan2022process}
J.~Fan and A.~Ruszczy{\'n}ski.
\newblock Process-based risk measures and risk-averse control of discrete-time
  systems.
\newblock {\em Mathematical Programming}, 191:113--140, 2022.

\bibitem{fei2022cascaded}
Y.~Fei and R.~Xu.
\newblock Cascaded gaps: Towards logarithmic regret for risk-sensitive
  reinforcement learning.
\newblock In {\em International Conference on Machine Learning}, pages
  6392--6417. PMLR, 2022.

\bibitem{fei2020risk}
Y.~Fei, Z.~Yang, Y.~Chen, Z.~Wang, and Q.~Xie.
\newblock Risk-sensitive reinforcement learning: Near-optimal risk-sample
  tradeoff in regret.
\newblock {\em Advances in Neural Information Processing Systems},
  33:22384--22395, 2020.

\bibitem{fei2021risk}
Y.~Fei, Z.~Yang, and Z.~Wang.
\newblock Risk-sensitive reinforcement learning with function approximation: A
  debiasing approach.
\newblock In {\em International Conference on Machine Learning}, pages
  3198--3207. PMLR, 2021.

\bibitem{frazier2008knowledge}
P.~I. Frazier, W.~B. Powell, and S.~Dayanik.
\newblock A knowledge-gradient policy for sequential information collection.
\newblock {\em SIAM Journal on Control and Optimization}, 47(5):2410--2439,
  2008.

\bibitem{gelman2013bayesian}
A.~Gelman, J.~B. Carlin, H.~S. Stern, D.~B. Dunson, A.~Vehtari, and D.~B.
  Rubin.
\newblock {\em Bayesian Data Analysis}.
\newblock CRC Press, Boca Raton, FL, 3rd edition, 2013.

\bibitem{gupta1996bayesian}
S.~S. Gupta and K.~J. Miescke.
\newblock Bayesian look ahead one-stage sampling allocations for selection of
  the best population.
\newblock {\em Journal of Statistical Planning and Inference}, 54(2):229--244,
  1996.

\bibitem{hoeffding1963probability}
W.~Hoeffding.
\newblock Probability inequalities for sums of bounded random variables.
\newblock {\em Journal of the American Statistical Association},
  58(301):13--30, 1963.

\bibitem{huang2017risk}
W.~Huang and W.~B. Haskell.
\newblock Risk-aware {Q}-learning for {M}arkov decision processes.
\newblock In {\em 2017 IEEE 56th Annual Conference on Decision and Control
  (CDC)}, pages 4928--4933. IEEE, 2017.

\bibitem{jin2018q}
C.~Jin, Z.~Allen-Zhu, S.~Bubeck, and M.~I. Jordan.
\newblock Is {Q}-learning provably efficient?
\newblock {\em Advances in Neural Information Processing Systems}, 31, 2018.

\bibitem{jin2020provably}
C.~Jin, Z.~Yang, Z.~Wang, and M.~I. Jordan.
\newblock Provably efficient reinforcement learning with linear function
  approximation.
\newblock In {\em Conference on Learning Theory}, pages 2137--2143. PMLR, 2020.

\bibitem{kose2021risk}
U.~K\"{o}se and A.~Ruszczy\'nski.
\newblock Risk-averse learning by temporal difference methods with {M}arkov
  risk measures.
\newblock {\em Journal of Machine Learning Research}, 22(1), 2021.

\bibitem{lamrisk}
T.~Lam, A.~Verma, B.~K.~H. Low, and P.~Jaillet.
\newblock Risk-aware reinforcement learning with coherent risk measures and
  non-linear function approximation.
\newblock In {\em The Eleventh International Conference on Learning
  Representations}, 2023.

\bibitem{lattimore2012pac}
T.~Lattimore and M.~Hutter.
\newblock {PAC} bounds for discounted {MDP}s.
\newblock In {\em ALT 2012}, pages 320--334. Springer, 2012.

\bibitem{lattimore2020bandit}
T.~Lattimore and C.~Szepesv{\'a}ri.
\newblock {\em Bandit algorithms}.
\newblock Cambridge University Press, 2020.

\bibitem{lin2013dynamic}
K.~Lin and S.~I. Marcus.
\newblock Dynamic programming with non-convex risk-sensitive measures.
\newblock In {\em American Control Conference (ACC), 2013}, pages 6778--6783.
  IEEE, 2013.

\bibitem{ma2018risk}
W.-J. Ma, C.~Oh, Y.~Liu, D.~Dentcheva, and M.~M. Zavlanos.
\newblock Risk-averse access point selection in wireless communication
  networks.
\newblock {\em IEEE Transactions on Control of Network Systems}, 6(1):24--36,
  2019.

\bibitem{melo2007q}
F.~S. Melo and M.~I. Ribeiro.
\newblock Q-learning with linear function approximation.
\newblock In {\em International Conference on Computational Learning Theory},
  pages 308--322. Springer, 2007.

\bibitem{modi2020sample}
A.~Modi, N.~Jiang, A.~Tewari, and S.~Singh.
\newblock Sample complexity of reinforcement learning using linearly combined
  model ensembles.
\newblock In {\em International Conference on Artificial Intelligence and
  Statistics}, pages 2010--2020. PMLR, 2020.

\bibitem{Prashanth2014}
L.~A. Prashanth and M.~Ghavamzadeh.
\newblock Actor-critic algorithms for risk-sensitive reinforcement learning.
\newblock {\em CoRR}, abs/1403.6530, 2014.

\bibitem{Ruszczynski2010Markov}
A.~Ruszczy{\'n}ski.
\newblock Risk-averse dynamic programming for {M}arkov decision processes.
\newblock {\em Mathematical Programming}, 125(2, Ser. B):235--261, 2010.

\bibitem{RuszczynskiShapiro2006a}
A.~Ruszczy\'nski and A.~Shapiro.
\newblock Optimization of convex risk functions.
\newblock {\em Mathematics of Operations Research}, 31(3):433--452, 2006.

\bibitem{Scandolo:2003}
G.~Scandolo.
\newblock {\em Risk Measures in a Dynamic Setting}.
\newblock PhD thesis, Universit\`{a} degli Studi di Milano, Milan, Italy, 2003.

\bibitem{shapiro2021lectures}
A.~Shapiro, D.~Dentcheva, and A.~Ruszczy\'nski.
\newblock {\em Lectures on {S}tochastic {P}rogramming: {M}odeling and
  {T}heory}.
\newblock SIAM, 2021.

\bibitem{shen2013risk}
Y.~Shen, W.~Stannat, and K.~Obermayer.
\newblock Risk-sensitive {M}arkov control processes.
\newblock {\em SIAM Journal on Control and Optimization}, 51(5):3652--3672,
  2013.

\bibitem{sidford2018near}
A.~Sidford, M.~Wang, X.~Wu, L.~Yang, and Y.~Ye.
\newblock Near-optimal time and sample complexities for solving {M}arkov
  decision processes with a generative model.
\newblock {\em Advances in Neural Information Processing Systems}, 31, 2018.

\bibitem{strehl2009reinforcement}
A.~L. Strehl, L.~Li, and M.~L. Littman.
\newblock Reinforcement learning in finite {MDP}s: {PAC} analysis.
\newblock {\em Journal of Machine Learning Research}, 10(11), 2009.

\bibitem{sutton1988learning}
R.~S. Sutton.
\newblock Learning to predict by the methods of temporal differences.
\newblock {\em Machine Learning}, 3(1):9--44, 1988.

\bibitem{Sutton1998}
R.~S. Sutton and A.~G. Barto.
\newblock {\em Reinforcement Learning: An Introduction}.
\newblock MIT Press, Cambridge, 1998.

\bibitem{Tamar2017}
A.~Tamar, Y.~Chow, M.~Ghavamzadeh, and S.~Mannor.
\newblock Sequential decision making with coherent risk.
\newblock {\em IEEE Transactions on Automatic Control}, 62(7):3323--3338, 2017.

\bibitem{Tamar2012}
A.~Tamar, D.~Di~Castro, and S.~Mannor.
\newblock Policy gradients with variance-related risk criteria.
\newblock In {\em Proceedings of the Twenty-Ninth International Conference on
  Machine Learning}, pages 387--396, 2012.

\bibitem{Tamar2014}
A.~Tamar, S.~Mannor, and H.~Xu.
\newblock Scaling up robust {MDPs} using function approximation.
\newblock In {\em International Conference on Machine Learning}, pages
  181--189, 2014.

\bibitem{tran2012knapsack}
L.~Tran-Thanh, A.~Chapman, A.~Rogers, and N.~R. Jennings.
\newblock Knapsack based optimal policies for budget-limited multi-armed
  bandits.
\newblock In {\em Proceedings of the AAAI Conference on Artificial
  Intelligence}, volume~26, pages 1134--1140, 2012.

\bibitem{tsitsiklis1997analysis}
J.~N. Tsitsiklis and B.~Van~Roy.
\newblock An analysis of temporal-difference learning with function
  approximation.
\newblock {\em IEEE Transactions on Automatic Control}, 42(5):674--690, 1997.

\bibitem{yang2020reinforcement}
Z.~Yang and M.~Wang.
\newblock Reinforcement learning in feature space: Matrix bandit, kernels, and
  regret bound.
\newblock In {\em International Conference on Machine Learning}, pages
  10746--10756. PMLR, 2020.

\bibitem{yin2022near}
M.~Yin, Y.~Duan, M.~Wang, and Y.-X. Wang.
\newblock Near-optimal offline reinforcement learning with linear
  representations: Leveraging variance information with pessimism.
\newblock In {\em International Conference on Learning Representation}, 2022.

\bibitem{yu2018approximate}
P.~Yu, W.~B. Haskell, and H.~Xu.
\newblock Approximate value iteration for risk-aware {M}arkov decision
  processes.
\newblock {\em IEEE Transactions on Automatic Control}, 63(9):3135--3142, 2018.

\end{thebibliography}





\end{document}